\documentclass[preprint,11pt]{article}
\usepackage{fullpage}
\ifdefined\usebigfont

\usepackage{boondox-cal}
\usepackage{times}
\usepackage[italicdiff]{physics}
\usepackage[fontsize=13pt]{scrextend}
\usepackage[left=1.56in,right=1.56in,top=1.74in,bottom=1.74in]{geometry}

\else
\fi

\usepackage{amssymb,amsfonts,amsmath,amsthm,amscd,dsfont,mathrsfs,bbm}
\usepackage{graphicx,float,psfrag,epsfig,amssymb}
\usepackage[usenames,dvipsnames,svgnames,table]{xcolor}
\definecolor{darkgreen}{rgb}{0.0,0,0.9}
\usepackage[pagebackref,letterpaper=true,colorlinks=true,pdfpagemode=none,citecolor=OliveGreen,linkcolor=BrickRed,urlcolor=BrickRed,pdfstartview=FitH]{hyperref}
\usepackage{wrapfig}
\usepackage{relsize}
\usepackage{color}
\usepackage{pict2e}
\usepackage{caption}
\usepackage{nameref}
\usepackage{makecell}
\usepackage[font={small}]{caption} 
\usepackage{subcaption}
\usepackage{mathtools}
\let\chapter\section
\usepackage[ruled,vlined]{algorithm2e}
\usepackage[noend]{algorithmic}

\DeclareMathAlphabet{\mathpzc}{OT1}{pzc}{m}{it}
\newtheorem{propo}{Proposition}[section]
\newtheorem{lemma}[propo]{Lemma}

\newtheorem{definition}[propo]{Definition}
\newtheorem{coro}[propo]{Corollary}
\newtheorem{thm}[propo]{Theorem}
\newtheorem{remark}[propo]{Remark}

\usepackage{hyperref}       % hyperlinks
\usepackage{url}   

\addtocontents{toc}{\protect\setcounter{tocdepth}{2}}

\newcommand{\twonorm}[1]{\left\|#1\right\|_{\ell_2}}
\def\sgn{\mathrm{sgn}}
\def\diag{{\mathrm diag}}

\def\tx{\widetilde{x}}

\def\cZ{{\cal Z}}
\def\cF{{\cal F}}

\def\cR{{\cal R}}

\def\reals{{\mathbb R}}

\def\eps{{\varepsilon}}
\def\prob{{\mathbb P}}
\def\E{{\mathbb E}}

\def\L0{{L_i}}

\def\de{{\rm d}}
\def\<{\langle}
\def\>{\rangle}
\def\diag{{\rm diag}}

\def\hth{\widehat{\theta}}

\def\F{{\sf F}}
\def\ind{{\mathbb I}}

\def \Tr{{\rm tr}}
\def\F{{\sf F}}
\def\normal{{\sf N}}

\def\sT{{\sf T}}

\def\sign{{\rm sign}}

\def\v*{v_i}
\def\T*{T_i}

\def\u*{u_i}
\def\F*{F_i}

\def\ty{\tilde{y}}
\def\tx{\tilde{x}}

\def\th{\theta}
\def\hth{{\widehat{\theta}}}

\def\diag{{\rm diag}}

\def\Q{\mathbb{Q}}
\def\Cpl{{\rm Cpl}}
\newcommand{\qnorm}[1]{\left\|#1\right\|_{\ell_q}}
\newcommand{\rnorm}[1]{\left\|#1\right\|_{\ell_r}}

\def\l1u{W}

\def\cV{\mathcal{V}}

\newcommand{\ajcomment}[1]{}

\makeatletter
\newcommand{\labitem}[2]{%
\def\@itemlabel{\text{#1}}
\item
\def\@currentlabel{#1}\label{#2}}
\makeatother

\DeclareMathAlphabet{\mathpzc}{OT1}{pzc}{m}{it}

\def\AR{\mathsf{AR}}
\def\SR{\mathsf{SR}}

\def\tell{\widetilde{\ell}}

\newcommand*\samethanks[1][\value{footnote}]{\footnotemark[#1]}

\author{%
Mohammad Mehrabi\thanks{Data Science and Operations Department, Marshall School of Business, University of Southern California} 
\quad\,
Adel Javanmard\samethanks[1] 
\quad\,
Ryan A. Rossi\thanks{Adobe Research}
\quad\,
Anup Rao\samethanks[2]
\quad\,
Tung Mai\samethanks[2]
}

\begin{document}
\title{Fundamental Tradeoffs in Distributionally Adversarial Training}
\maketitle
\begin{abstract}
Despite the wide empirical success of machine learning algorithms, it is well known that these powerful algorithms can perform poorly on adversarially manipulated data.
Adversarial training is among the most effective techniques to improve the robustness of models against adversarial perturbations. However, the full effect of this approach on models is not well understood. For example, while adversarial training can reduce the adversarial risk (prediction error against an adversary), it sometimes increase standard risk (generalization error when there is no adversary). Even more, such behavior is impacted by various elements of the learning problem, including the size and quality of training data, specific forms of adversarial perturbations in the input, model overparameterization, and adversary's power, among others. In this paper, we focus on \emph{distribution perturbing} adversary framework wherein the adversary can change the test distribution within a neighborhood of the training data distribution. The neighborhood is defined via Wasserstein distance between distributions and the radius of the neighborhood is a measure of adversary's manipulative power. We study the tradeoff between standard risk and adversarial risk and derive the Pareto-optimal tradeoff, achievable over specific classes of models, in the infinite data limit with features dimension kept fixed. We consider three learning settings:
1) Regression with the class of linear models; 2) Binary classification under the Gaussian mixtures data model, with the class of linear classifiers; 3) Regression with the class of random features model (which can be equivalently represented as two-layer neural network with random first-layer weights).
We show that a tradeoff between standard and adversarial risk is manifested in all three settings. 
We further characterize the Pareto-optimal tradeoff curves and discuss how a variety of factors, such as features correlation, adversary's power or the width of two-layer neural network would affect this tradeoff. 
\end{abstract}

%\ryan{Minor suggestion: features correlation => correlation between features?}

\section{Introduction}
 Modern machine learning algorithms, and in particular deep neural networks, have demonstrated breakthrough
empirical performance, and have been deployed in a multitude of applications domains ranging from visual object classification to
speech recognition, robotics, natural language processing and healthcare.
 The common practice to train these models is by empirical loss minimization on the training data. Nonetheless, it has been 
observed that the resulting models are surprisingly vulnerable to minute discrepancies between the test and the training data distributions. There are several well documented examples of such behavior in computer vision and image processing where small imperceptible manipulations of images can significantly compromise the performance of the state-of-the-art classifiers~\cite{szegedy2013intriguing,biggio2013evasion}.
Other examples include well-designed malicious content like malware which can be labeled legitimate by the classifier and harm the system~\cite{chen2017adversarial,papernot2017practical}, or adversarial attacks on speech recognition systems, such as GoogleNow or Siri, which consists in voice commands that are incomprehensible or even completely inaudible to human and can still 
control the systems~\cite{carlini2016hidden,vaidya2015cocaine,zhang2017dolphinattack}. It is evident that in practice such vulnerability can have catastrophic consequences.

By studying adversarial samples, one can in turn improve the robustness of machine learning algorithms against adversarial attacks. In the past few years, there has been a significant research on generating various adversarial samples~\cite{carlini2017adversarial,athalye2018obfuscated,goodfellow2014explaining,papernot2016crafting} and defenses~\cite{DBLP:conf/iclr/MadryMSTV18,cisse2017parseval,papernot2016distillation}.  Among the considerable effort to improve the adversarial robustness of algorithms, adversarial training is one of the most effective techniques. Adversarial training is often formulated as a minimax optimization
problem, where the inner maximization aims to find an adversarial example that maximizes 
a predictive loss function, while the outer minimization aims to train a robust estimator
using the generated adversarial examples~\cite{DBLP:journals/corr/GoodfellowSS14, kurakin2016adversarial, DBLP:conf/iclr/MadryMSTV18, DBLP:conf/iclr/RaghunathanSL18, DBLP:conf/icml/WongK18}.

While adversarial training techniques have been successful in improving the adversarial robustness of the models, their full effect on machine learning systems is not well understood. In particular, some studies~\cite{DBLP:conf/iclr/MadryMSTV18} observed that the robustness virtue of adversarial training comes at the cost of worsening the performance on natural unperturbed inputs, i.e, increasing generalization error. However,~\cite{tsipras2018robustness} observes empirically that when there are very
few training data, adversarial training can help with reducing the generalization error.  Complicating matters
further,~\cite{raghunathan2019adversarial} argues that additional unlabeled data can mitigate the tension between adversarial risk (predictive performance against adversarial perturbations) and the standard risk (predictive performance when there is no adversary, a.k.a generalization error). These observations raise the following important question regarding adversarial training:
\begin{quote}
\emph{Is there a `fundamental' tradeoff between adversarial risk and standard risk? Or do there exist models that are optimal with respect to both of these measures?
What are the roles of different factors, such as adversary's power, problem dimension and the complexity of the model class (e.g., number of neurons) in the interplay between standard risk and adversarial risk?} 
\end{quote}
Here, by `fundamental tradeoff' we mean a tradeoff that holds given unlimited computational power and infinite training data to train a model. In this work, we answer these questions for adversarial distribution shifts, where the adversary can shift the test data distribution, making it different from the training data distribution.
The test data distribution can be an arbitrary but fixed distribution in a neighborhood of the training data distribution and the radius of this neighborhood is in fact a measure of 
adversary's power.  

\medskip
\noindent{\bf Contributions.} We next summarize our contributions in this paper:

\begin{itemize}
\item We characterize the fundamental tradeoff between standard risk and adversarial risk for distributionally adversarial training for the settings of linear regression and binary classification (under a Gaussian mixtures model). We focus on infinite data limit $(n\to \infty)$ with finite feature dimension $(d)$ and hence our analysis is at population level. 
The fundamental tradeoff is characterized by studying the Pareto optimal fronts for the achievability region in the two dimensional standard risk-adversarial risk region. 
The Pareto optimal front consists in the set of estimators for which one cannot decrease both standard and adversarial risk by deviating from these estimators.
Similar tradeoffs have been derived for linear regression setting with norm bounded adversarial perturbation and isotropic Gaussian features~\cite{pmlr-v125-javanmard20a}. Here we focus on distribution perturbing adversaries and consider general anisotropic Gaussian features. 

\item For the binary classification we consider a Gaussian mixtures model with general feature covariance and a distribution perturbing adversary, where the perturbation is measured in terms of the Wasserstein metric with general $\ell_r$ norm. (We refer to Sections~\ref{sec:background} and~\ref{sec:binary} for further details and formal definitions). Our analysis shows how the fundamental tradeoff between standard and adversarial risk is impacted by a variety of factors, such as adversary's power, feature dimension, features correlation and the choice of $\ell_r$ perturbation norm. 
An interesting observation is that for $r=2$ the tradeoff between standard and adversarial risk vanishes. In other words, there exists a model which achieve both the optimal standard risk and the optimal adversarial risk.    
% forward pointer??

\item We also study the Pareto optimal tradeoffs between the standard and adversarial risks for the problem of learning an \emph{unknown function} over the $d$-dimensional sphere using random features model. This can be represented as linear models with $N$ random nonlinear features of the form $\sigma(w_a^\sT x)$, $1\le a\le N$, with $\sigma(\cdot)$ a nonlinear activation. Equivalently this can be characterized as fitting a two-layer neural network with random first-layer.  Building upon approximation formula for adversarial risk, we study the effect of network width $N$ on the tradeoff between standard and adversarial risks.

%\item The Pareto optimal tradeoffs above are derived by considering the class of all possible linear hypotheses (linear regressor models for the linear regression setting and linear classifiers for the binary classification setting.) A tantalizing problem is whether one can break this tradeoff by considering a more complex class of models. Taking the first steps to answer this question, we consider a linear regression setting describing the relation between response $y$ and feature $x$ and then fit a two-layer neural network with random first-layer. This can equivalently be described as a regression model with $N$ random nonlinear features of the form $\sigma(w_a^\sT x)$, $1\le a\le N$, with $\sigma(\cdot)$ the ReLU activation. Building upon approximation formula for adversarial risk, we study the effect of network width $N$ on the tradeoff between standard and adversarial risks.

\end{itemize}

\subsection{Further related work} Very recent work~\cite{javanmard2020precise,taheri2020asymptotic} have focused on binary classification, under Gaussian mixtures model and proposed a precise characterization of the standard and adversarial risk achieved by a specific class of adversarial training approach~\cite{tsipras2018robustness,madry2017towards}. These work consider an asymptotic regime where the sample size grows in proportion to the problem dimension $d$ and focus on norm bounded adversarial perturbation. In comparison, we consider a fixed $d$, infinite $n$ setting and consider distribution perturbing adversary. Also we focus on fundamental tradeoffs achieved by any linear classifier, while~\cite{javanmard2020precise,taheri2020asymptotic} work with a specific class of saddle point estimator. 
The other work~\cite{dobriban2020provable} also considers norm bounded adversarial perturbation for the classification problem and studies
the optimal $\ell_2$ and $\ell_\infty$ robust linear classifiers  assuming access to the class
averages. Furthermore, it also studies the tradeoff between standard and robust accuracies from a Bayesian perspective by contrasting
this optimal robust classifier with the Bayes optimal classifier in a non-adversarial setting.

%
%
%
% 

%\tableofcontents
\section{Problem formulation}
%\subsection{Introduction}
In a classic supervised learning setting, a learner is given $n$ pairs of data points $\{z_i:=(x_i,y_i)\}_{i=1:n}$ with $x_i\in\reals^d$ representing features vectors and $y_i$ the response variables (or labels). The common assumption in supervised machine learning is that the data points $z_i$ are drawn independently and identically from some probability measure $\prob_Z$ defined over the space $\mathcal{Z}:=\mathcal{X}\times \mathcal{Y}$. Given this training data, the learner would like to fit a parametric 
function $f_{\th}$ with $\th \in \reals^d$ to predict the response (label)  on new points $x$. 

A common practice to model fitting is through the empirical risk minimization:
 \begin{equation}\label{eq1}
     \hth=\arg\min\limits_{\th \in \reals^d}^{} \frac{1}{n}\sum\limits_{j=1}^{n}\ell(\th;(x_j,y_j))\,, 
 \end{equation}
with $\ell(\th;(x,y)):=\tell(f_{\th}(x),y)$ and $\tell$ being a loss function which captures the discrepancy between the estimated value $f_{\th}(x)$ and the true response value $y$. 
The performance of the model is then measured in terms of \emph{standard risk} (a.k.a. generalization error), defined as
\begin{align}
\SR(\th) : = \E_{z=(x,y)\sim \prob_Z}\left[\ell(\th;(x,y))\right]\,.
\end{align} 
Standard risk is a population risk and quantifies the expected error on new data points drawn from the same distribution as the training data.

Although the empirical risk minimization is a widely used approach for model learning, it is well known that the resulting models can be highly vulnerable to adversarial perturbations of their inputs, known as adversarial attacks. In fact, seemingly small indiscernible changes to the input feature can significantly degrade the predictive performance of the model.

We next discuss the adversarial setting and two common adversary models that are considered in literature.

%Several adversarial training algorithms have been proposed to improve the robustness of models against adversarial attacks, and while being successful in this regard, it has been observed that this benefit often comes at the cost of losing their accuracy on non-adversarial inputs. Put differently, decreasing adversarial risk, the expected loss on adversarially perturbed test data, often leads to an increase in the standard risk of the model. \aj{TBC}

\subsection{Adversarial setting}\label{sec:adv-setting}
The adversarial setting can be perceived as a game between the learner and the adversary. Given access to the training data, drawn i.i.d from a common distribution $\prob_Z$, the learner chooses a model $\theta$. Depending on the adversary's budget $\eps$, the adversary chooses a test data point $(\tilde{x},\tilde{y})$ that can deviate from a typical test point according to one of the following models. The performance of the model $\theta$ is then measured in terms of predicting $\tilde{y}$ given the perturbed input $\tilde{x}$.  

%A test data $(x,y)$ is drawn independently from $\prob_Z$ and the adversary perturbs it to $(\tilde{x},\tilde{y})$.
%There are two common types of perturbations which we discuss next.

\smallskip

\noindent{\bf Norm-bounded perturbations.} In this setting, $\tilde{y} =y$ (no perturbation on the response)
and $\tilde{x} = x+\delta$ where $\delta$ can be an arbitrary vector from $\ell_r$-ball of radius $\eps$. The \emph{adversarial risk} in this case is defined as
\begin{align}
\AR(\th):= \E_{(x,y)\sim \prob_Z} \left[ \sup_{\rnorm{\delta}\le \eps} \ell(\th;(x+\delta, y)) \right]\,.
\end{align}
\smallskip

\noindent{\bf Distribution shift.} In this setting, the adversary can shift the distribution of test data, making it different than the training distribution $\prob_Z$. Specifically, $(\tilde{x},\tilde{y})\sim \Q$ where $\Q\in \mathcal{U}_\eps(\prob_Z)$ denotes an $\eps$- neighborhood of the distribution $\prob_Z$. A popular choice of this neighborhood is via the Wasserstein distance, which is formally defined below. In this case, the \emph{adversarial risk} is defined as
\begin{align}\label{eq:U}
\AR(\th):= \sup_{\Q\in \mathcal{U}_\eps(\prob_Z)}\; \E_{(\tilde{x},\tilde{y})\sim \Q} \left[\ell(\th;(\tilde{x}, \tilde{y})) \right]\,.
\end{align}

 Note that this is a strong notion of adversary as the perturbation is chosen \emph{after} observing both the model $\theta$ and data point $(x,y)$ (in norm-bounded perturbation model) or the training data distribution $\prob_Z$ (in the distribution shift model). 
 
 Our primary focus on this work is on the distribution shift adversary model with Wasserstein metric to measure the distance between distributions. The next section provides a  brief background on the Wasserstein robust loss which will be used later in our work.  
 %In Section~\ref{???} we also discuss the relation between the adversary models discussed above.  

%The defined risk measure in equation \ref{eq1} is a special case of population risk $\cR(\th):=\E_{(X,Y)\sim P_Z}\left[\ell(\th;(X,Y))\right]$ where we used the empirical law $\hat{P}_Z=\frac{1}{n}\sum\limits_{j=1}^{n}\delta_{z_i}$ as a surrogate for unseen population measure $P_Z$. A growing number of works show that 
%the obtained model from optimization \ref{eq1} will have poor performance on adversarial attacks in which the test data point is perturbed under an adversarial model. Previous lines of work consider the adversarial model which allows for additive perturbations within a bounded certain set of values e.g. $S=\{\delta \in \reals^d: ||\delta|| \leq \alpha\}$. \\
%We will focus on a setting inspired from \textit{distributionally robust optimization} that will hedge against adversarial perturbations by choosing a model with smallest worst population loss on a Wasserstein ball of probability measures around the actual data generating law $P_Z$.
%
%\begin{align*}
% \th =\arg\min\limits_{\th \in \reals^d}^{} \sup\limits_{Q \in \mathcal{U}}^{} \E_{Z\sim Q}\left[\ell(\th;Z)\right ]  
% \end{align*}
%
% Our main goal is to understand tradeoffs between robust and standard accuracy with adversarial training.
\subsection{Background on Wasserstein robust loss}\label{sec:background}
Let $\cZ$ be a metric space endowed with metric $d:\cZ\times \cZ \rightarrow \reals_{\geq 0}$. Denote by $\mathcal{P}(\cZ)$ the set of all Borel probability measures on $\cZ$. For a $\Q$-measurable function $f$, the $\mathcal{L}^p(\Q)$-norm of $f$ is defined as
\begin{align}\label{eq:LQ-norm}
\|f\|_{\Q,p} := \begin{cases}
\left(\int_{\mathcal{Z}} |f|^p\;\de \Q \right)^{1/p}\,\quad &\text{ for }p\in[1,\infty)\\
\underset{z\in \mathcal{Z}}{\Q{\text{-ess sup}}}\;\; |f(z)| &\text{ for }p=\infty
\end{cases}
\end{align}
For two distributions $\prob,\Q\in \mathcal{P}(\cZ)$ the Wasserstein distance of order $p$ is given by
\begin{align}\label{eq:Wp-def}
W_p(\prob,\Q):= \inf_{\pi\in\Cpl(\prob,\Q)} \|d\|_{\pi,p}\,,
\end{align}
where the coupling set $\Cpl(\prob,\Q)$ denotes the set of all probability measures $\pi$ on $\mathcal{Z}\times
\mathcal{Z}$ with the first marginal $\pi_1:= \pi(\cdot\times \mathcal{Z}) = \prob$ and the second
marginal $\pi_2:= \pi(\mathcal{Z} \times \cdot) = \Q$.

We use the Wasserstein distance to define the neighborhood set $\mathcal{U}_\eps$ in the distribution shift adversary model. Namely,
\begin{align}\label{eq:U0}
\mathcal{U}_\eps(\prob_Z) : = \{\Q\in \mathcal{P}(\cZ):\; W_p(\prob_Z,\Q) \le \eps\}\,.
\end{align}
 In this case we refer to $\AR(\th)$ given by \eqref{eq:U} as \emph{Wasserstein adversarial risk}. Note that this notion involves a maximization over distributions $\Q\in \mathcal{P}(\cZ)$ which can be daunting. However, an important result from distributional robust optimization which we also use in our characterization of $\AR(\th)$ is that the strong duality holds for this problem under general conditions. The dual problem of \eqref{eq:U} is given by
 \begin{align}\label{eq:dual} 
 \begin{cases} 
 \min_{\gamma\ge0} \Big\{\gamma \eps^p + \E_{\prob_Z} [\phi_\gamma(\theta;z)] \Big\}\,,\quad &p\in[1,\infty),\\
 \E_{z\sim \prob_Z} \Big[\sup_{\tilde{z}\in \mathcal{Z}} \{\ell(\th;\tilde{z}):\; d(z,\tilde{z})\le\eps \}   \Big] &p=\infty\,.
 \end{cases}
 \end{align}
Here $\phi_{\gamma}(\th;z)$ is the robust surrogate for the loss function $\ell(\th;z)$ and is defined as
\begin{align}\label{eq:phi}
\phi_{\gamma}(\th;z_0):=\sup\limits_{z\in\cZ}^{} \{  \ell(\th;z)-\gamma d^p(z,z_0)  \} \,. 
\end{align}
 For $p\in[1,\infty)$ it is shown that strong duality holds if either $\prob_Z$ has finite support or $\mathcal{Z}$ is a Polish space~\cite{gao2016distributionally}. For $p=\infty$,~\cite[Lemma EC.2]{gao2017wasserstein} also shows that strong duality holds if $\prob_Z$ has finite support.
 
 \subsubsection{Regularization effect of Wasserstein adversarial loss}
 It is clear from the definition that $\AR(\th)\ge \SR(\th)$ for any model $\th$. Understanding the tradeoff between standard and adversarial risks is intimately related to the gap $\AR(\th)-\SR(\th)$.  The gap between the Wasserstein adversarial loss and the standard loss has been studied in several settings in the context of distributionally robust optimization (DRO)~\cite{bartl2020robust,gao2017wasserstein}. In particular,~\cite{bartl2020robust,gao2017wasserstein} introduced the notion of \emph{variation of the loss}, denoted as $\cV(\ell)$, as a measure of the magnitude change in the expected loss when the data distribution is perturbed, and showed that the Wasserstein adversarial loss is closely related to regularizing the nominal loss by the variation $\cV(\ell)$ regularizer. The formal definition of the variation of loss, recalled from~\cite{gao2017wasserstein}, is given below.
\begin{definition} (Variation of the loss). Suppose that $\mathcal{Z}$ is a normed space with norm $\|\cdot\|$. Let $\ell$ be a continuous function on $\mathcal{Z}$. Also assume that $\nabla_z \ell$ exists $\prob$-almost everywhere. The variation of loss $\ell$ with respect to $\prob$ is defined as
\begin{align}
\cV_{\prob,q}(\ell): = \begin{cases}
\|\|\nabla_z\ell \|_*\|_{\prob,q}\, \quad &q\in [1,\infty)\,,\\
\underset{z\in \mathcal{Z}}{\prob_Z{\mathrm{-ess sup}} }\; \underset{\tilde{z}\neq z}{\sup}\; \frac{(\ell(\tilde{z})-\ell(z))_+}{\|\tilde{z}-z\|}, & q= \infty\,.
\end{cases}
\end{align}
Here $\|\cdot\|_*$ denotes the dual norm of $\|\cdot\|$ and we recall that $\|\cdot\|_{\prob,q}$ is the $\mathcal{L}^q(\prob)$-norm given by~\eqref{eq:LQ-norm}.
\end{definition}
The following proposition from~\cite{bartl2020robust,gao2017wasserstein} states that the variation of loss captures the first order term of the gap between Wasserstein adversarial risk and standard risk for small $\eps$.
 \begin{propo}\label{pro:approx}
Suppose that the loss $\ell(\th;z)$ is differentiable in the interior of $\mathcal{Z}$ for every $\th$, and $\nabla_z\ell$ is continuous on $\mathcal{Z}$. When $p\in(1,\infty)$, assume that there exists $M,L\ge 0$ such that for every $\th$ and $z,\tilde{z}\in \mathcal{Z}$,
\[
\|\nabla_z\ell(\th;\tilde{z}) - \nabla_z\ell(\th;z)\|_*\le M+L \|\tilde{z}- z\|^{p-1}\,.
\]
When $p=\infty$, assume instead that there exists $M\ge0$ and $\delta_0>0$ such that for every $\th$ and $z,\tilde{z}\in \mathcal{Z}$ with $\|\tilde{z}-z\|<\delta_0$, we have
\[
\|\nabla_z \ell(\th;\tilde{z}) - \nabla_z \ell(\th;z)\|_*\le M\,.
\]
Then, there exists $\bar{\eps}$ such that for all $0\le \eps<\bar{\eps}$ and all $\theta$
\begin{align}\label{eq:gap}
\AR(\th) -\SR(\th) = \eps \cV_{\prob_Z,q}(\ell) + O(\eps^2)\,,
\end{align}
where $\frac{1}{p} + \frac{1}{q} = 1$ and $p$ is the order of Wasserstein distance in defining set $\mathcal{U}_\eps(\prob_Z)$ in the adversarial risk~\eqref{eq:U}. 
 \end{propo}
  
By virtue of Proposition~\ref{pro:approx}, the Wasserstein adversarial risk can be perceived as a regularized form of the standard risk with regularization given by the variation of the loss. Nonetheless, note that this is only an approximation which captures the first order terms for small adversary's power $\eps$. (See also~\cite[Remark 8]{bartl2020robust} for an upper bound on the gap up to second order terms in $\eps$.) 

In this paper, we consider the settings of linear regression and binary classification. For these settings, only for the special case of $p=1$ (1-Wasserstein) and when the loss is Lipschitz and its derivative converges at $\infty$, it is shown that the gap~\eqref{eq:gap} is linear in $\eps$ and therefore is precisely characterized as $\eps \cV_{\prob_Z,q}(\ell)$. However, as we will consider more common losses for these settings, namely quadratic loss for linear regression and 0-1 loss for classification, such characterization does not apply to our settings and requires a direct derivation of adversarial risk. Later, in Section~\ref{sec:nonlinear} we use the result of Proposition~\ref{pro:approx}  to study the tradeoff between $\SR$ and $\AR$
in the problem of learning an unknown function over the $d$-dimensional sphere $\mathbb{S}^{d-1}$. 
%the role of model size (number of neurons) on  in the presence of a weak adversary. \aj{Up to here}
  
%For two probability measures $P,Q \in \mathcal{P}(\cZ)$  and $p \in [1,+\infty)$ let the $p$-Wasserstein distance between $P$ and $Q$ be  
%$$ \mathcal{W}_{p}(P,Q)=\left(\inf\limits_{\gamma \in \Pi(P,Q)}^{}\int\limits_{\cZ\times \cZ}^{} d^p(z_1,z_2)\gamma(dz_1,dz_2) \right)^{\frac{1}{p}}\,,$$
%where $\Pi(P,Q)$ is a set of all coupling measures $M$ of $P$ and $Q$ meaning for every measurable set $A$, we have $M(A,\cZ)=P(A)$ and $M(\cZ,A)=Q(A)$.
%For every $P\in \cP(\cZ)$ define the $\alpha$-Wasserstein ball centered at $P$ as  
%$$ \mathfrak{M}^{\alpha}_{p}(P):= \{ Q: \mathcal{W}_p(Q,P) \leq \alpha \}$$
\section{Main results}
In this paper we focus on the distribution perturbing adversary and aim at understanding the fundamental tradeoff between standard risk and adversarial risk, which holds regardless of computational power or the size of training data. 
We consider 2-Wasserstein distance ($p=2$) with the metric $d(z,\tilde{z})$ defined as
\begin{align}\label{eq:distance}
d(z,\tilde{z}) = \twonorm{x-\tilde{x}} + \infty\cdot \ind\{y\neq\tilde{y}\}\,,\quad z=(x,y), \quad \tilde{z} = (\tilde{x},\tilde{y})\,,
\end{align}
Therefore, the adversary with a finite power $\eps$ can only perturb the distribution of the input feature $x$, but not $y$. Otherwise, the distance $d(z,\tilde{z})$ becomes infinite and so the Wasserstein distance between the the data distribution $\prob_Z$ and the adversary distribution $\Q$, given by~\eqref{eq:Wp-def}, also becomes infinite. It is worth noting that this choice of $d$ is only for simplicity of presentations and our results in this section can be derived in a straightforward manner for distances that also allow perturbations on the $y$ component.

The following remark relates the two types of adversary discussed in Section~\ref{sec:adv-setting} and follows readily from the definition~\eqref{eq:Wp-def} and Equation~\eqref{eq:LQ-norm}. 
\begin{remark}
For distance $d(\cdot,\cdot)$ given by~\eqref{eq:distance}, the adversary model with norm bounded perturbations correspond to the distribution shifting adversary model with $p=\infty$ Wasserstein distance.  
\end{remark}

\subsection{Linear regression}
We consider the class of linear models to fit to data with quadratic loss $\ell(z;\th) = (y-x^\sT\th)^2$. Our first result is a closed form representation of the Wasserstein adversarial risk~\eqref{eq:U} in this case.
\begin{propo}\label{propo:AR-linear}
Consider the quadratic loss $\ell(z;\th) = (y-x^\sT\th)^2$ and the distribution perturbing adversary with $\mathcal{U}_\eps(\prob_Z)$ given by~\eqref{eq:U0} with $p=2$ and the metric $d$ given by~\eqref{eq:distance}. In this case the adversarial risk $\AR(\th)$ admits the following form:
\begin{align}\label{eq:AR-linear}
\AR(\th) = \left(\sqrt{\E_{\prob_Z} [(y-x^\sT\th)^2]} +\eps \twonorm{\th}\right)^{2}\,.
\end{align}
\end{propo}
To prove Proposition~\ref{propo:AR-linear} we exploit the dual problem~\eqref{eq:dual}. We refer to Section~\ref{proof:AR-linear} for the proof of Proposition~\ref{propo:AR-linear}.  
\medskip
 
\noindent{\bf Pareto optimal curve.} For the linear regression setting, note that the standard risk $\SR(\th)$ and the adversarial risk $\AR(\th)$ are convex functions of $\theta$. (The latter is convex since $\E_\Q[(y-x^\sT\th)^2]$ is convex for any distribution $\Q$ and maximization preserves convexity.) Therefore, we can find (almost) all Pareto optimal points by minimizing a weighted combination of the two risk measures by varying the weight $\lambda$:
\begin{align}\label{eq:weightedSum}
\theta_\lambda : =\arg\min_{\theta} \lambda \SR(\theta) + \AR(\theta)
\end{align}

The Pareto optimal curve is then given by $\{(\SR(\th_\lambda),\AR(\th_\lambda)): \lambda\ge 0\}$.

\begin{thm}\label{thm:pareto-linear}
Consider the setting of Proposition~\ref{propo:AR-linear} with $v:= \E[yx]$, $\sigma^2_y:=\E[y^2]$, and $\Sigma:= \E[xx^\sT]$.
Then the solution $\theta$ of optimization~\eqref{eq:weightedSum} is given either by (i) $\theta_\lambda = 0$ or $(ii)$ $\theta_\lambda = (\Sigma+\gamma_* I)^{-1}v$,
with $\gamma_*$ the fixed point of the following two equations:
\begin{align}
\gamma &= \frac{\eps^2+\eps A}{1+\lambda +\dfrac{\eps}{A}}\,, \label{eq:linear-fixedpoints1}\\ 
A &= \frac{1}{\twonorm{(\Sigma+\gamma I)^{-1}v}} 
\left(\sigma_y^2 + \twonorm{\Sigma^{1/2}(\Sigma+\gamma I)^{-1}v}^2 - 2v^\sT(\Sigma+\gamma I)^{-1} v\right)^{1/2}\,. \label{eq:linear-fixedpoints2} 
\end{align}
In case $(i)$ we have $\SR(\th_\lambda) =\AR(\th_\lambda) =\sigma_y^2$. In case $(ii)$ we have
\begin{eqnarray}
\begin{split}\label{eq:SR-AR_linear}
\SR(\th_\lambda) &= A_*^2 \twonorm{(\Sigma+\gamma_*I)^{-1}v}^2\,,\\
\AR(\th_\lambda) &= (A_*+\eps)^2 \twonorm{(\Sigma+\gamma_*I)^{-1}v}^2\,,
\end{split}
\end{eqnarray}
where $A_*$ is given by~\eqref{eq:linear-fixedpoints2} when $\gamma=\gamma_*$.
\end{thm}
The proof of Theorem~\ref{thm:pareto-linear} is given in Section~\ref{proof:thm:pareto-linear}. 

\begin{coro}\label{coro:pareto-linear}
Suppose that data is generated according to linear model $y = x^\sT \th_0+w$ with $w\sim\normal(0,\sigma^2)$ and isotropic features satisfying $\E[xx^\sT] = I_d$. 
Then the solution $\theta_\lambda$ of optimization~\eqref{eq:weightedSum} is given either by (i) $\theta_\lambda = 0$ or $(ii)$ $\th_\lambda = (1+\gamma_*)^{-1}\th_0$, where $\gamma_*$ is the fixed point of the following two equations:
\begin{align}
\gamma &= \frac{\eps^2+\eps A}{1+\lambda +\dfrac{\eps}{A}}\,,\\
A &=  \left(\gamma^2 + (1+\gamma)^2 \frac{\sigma^2}{\twonorm{\th_0}^2} \right)^{1/2}\,. \label{eq:linear:coro:A}
\end{align}
In case $(i)$ we have $\SR(\th_\lambda) =\AR(\th_\lambda) =\sigma^2+\twonorm{\th_0}^2$. In case $(ii)$ we have
\begin{align}
\SR(\th_\lambda) &= A_*^2 (1+\gamma_*)^{-2}\twonorm{\th_0}^2\,,\\
\AR(\th_\lambda) &= (A_*+\eps)^2(1+\gamma_*)^{-2}\twonorm{\th_0}^2\,,
\end{align}
where $A_*$ is given by~\eqref{eq:linear:coro:A} when $\gamma=\gamma_*$.
\end{coro}
The proof of Corollary~\ref{coro:pareto-linear} is provided in Section~\ref{proof:coro-linear}.

Figure \ref{fig:linear} shows the effect of various parameters on the Pareto optimal tradeoffs between adversarial ($\AR$) and standard risks ($\SR$) in linear regression setting. We consider data generated according to the linear model $y = x^\sT \th_0+w$ with $w\sim\normal(0,1)$ and features $x_i$ sampled i.i.d from $\normal(0,\Sigma)$ where $\Sigma_{i,j}=\rho^{|i-j|}$. Figure \ref{fig:linear:d} demonstrates the role of features dimension $d$ on the Pareto optimal curve for the setting with $\rho=0$ (identity covariance matrix), adversary's power $\varepsilon=1$, and the entries of $\th_0$ generated independently from $\normal(0,1/40)$. Note that by Corollary \ref{coro:pareto-linear}, in the case of isotropic features, standard risk and adversarial risks depend on $\th_0$ only through its $\ell_2$ norm.  The variations in the Pareto-curve here is due to variations in $\twonorm{\th_0}$ as $d$ changes.   

Figure \ref{fig:linear:rho} investigates the role of dependency across features ($\rho$) in the optimal tradeoff between standard and adversarial risks. In this setting $d=10$, $\eps=1$, and $\th_0\sim\frac{1}{\sqrt{d}} \normal(0,I)$. As we see all the curves start from the same point. This can be easily verified by the result of Theorem~\ref{thm:pareto-linear}: For the linear data model $y = x^\sT \th_0+ w$, we have $v = \Sigma \th_0$ and at $\lambda=\infty$, the Pareto-optimal estimator is the minimizer of the standard risk, i.e. $\th_{\lambda=\infty} = \th_0$. Also by~\eqref{eq:linear-fixedpoints1} we have $\gamma_*= 0$, and by \eqref{eq:linear-fixedpoints2} we obtain $A = \sigma/\twonorm{\th_0}$. Plugging these values in~\eqref{eq:SR-AR_linear} we get $\SR(\th_\infty) = \sigma^2$ and $\AR(\th_\infty) = (\sigma+\eps_0\twonorm{\th_0})^2$. Therefore both metrics become independent of $\rho$ at $\lambda=\infty$.

Also looking at the right-most point of the Pareto-curves, corresponding to $\lambda=0$, we see that as $\rho$ increases from small to moderate values, this point moves upward-right, indicating that both standard and adversarial risks increase, but after some value of $\rho$, we start to see a reverse behavior, where standard and adversarial risks start to decrease. 

Finally in Figure \ref{fig:linear:eps} we observe the effect of adversary's budget $\varepsilon$ on the Pareto optimal curve. Here, $d=10$, $\rho=0$, and $\th_0\sim\frac{1}{\sqrt{d}} \normal(0,I)$. Clearly, as $\eps$ grows there is a wider range of Pareto-optimal estimators and the two measures of risks would deviate further from each other.
When $\eps$ becomes smaller, the two measures of standard and adversarial risks get closer to each other and so the Pareto-optimal curve becomes shorter. 
\begin{figure}
	\centering
	\begin{subfigure}[b]{0.32\textwidth}
		\centering
	\includegraphics[scale=0.31]{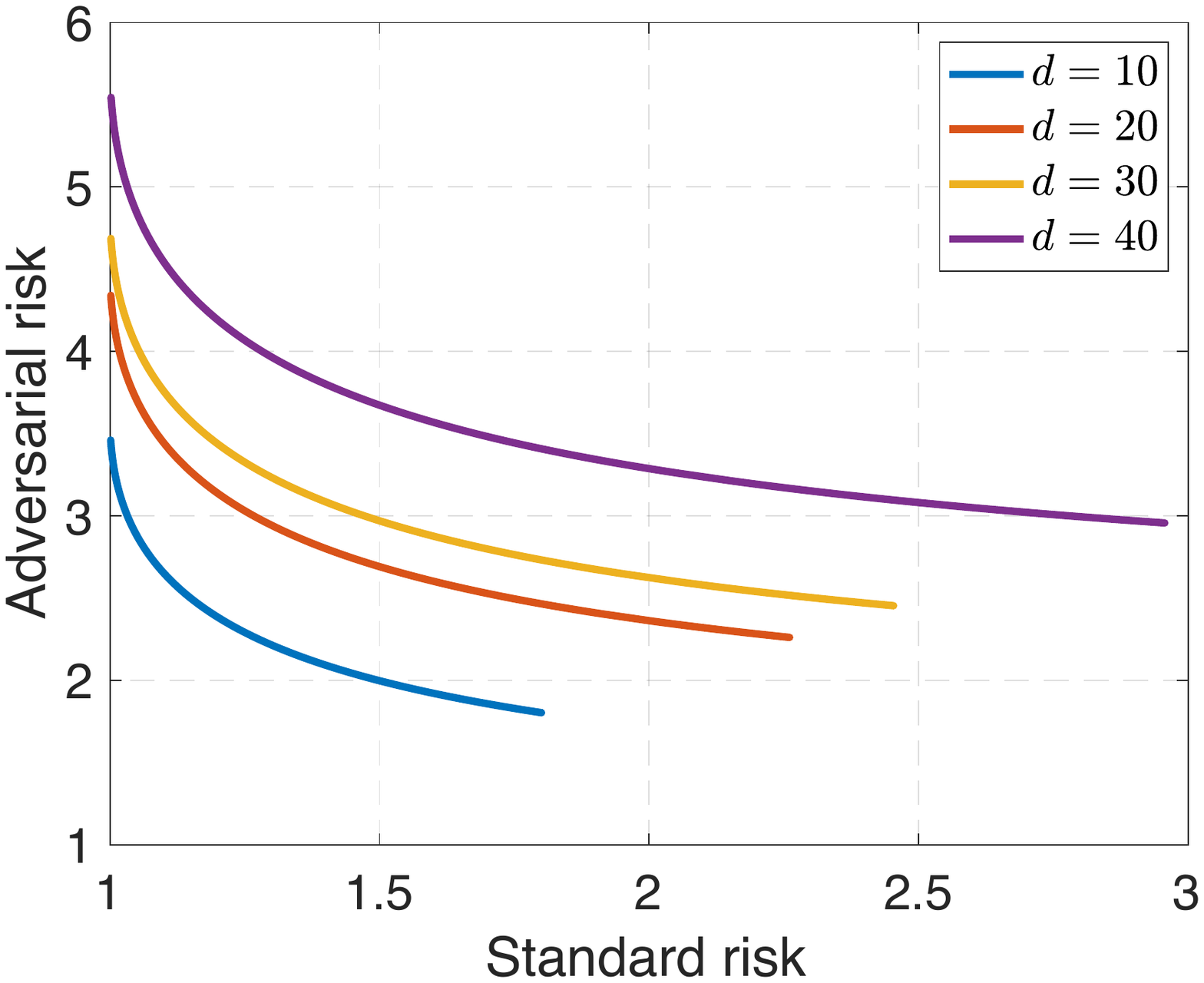}
\caption{Pareto optimal curve for several feature dimensions $d$ with $\rho=0$ and $\eps=1$. }
	\label{fig:linear:d}
	\end{subfigure}
\hfill
\begin{subfigure}[b]{0.32\textwidth}
	\centering
	\includegraphics[scale=0.31]{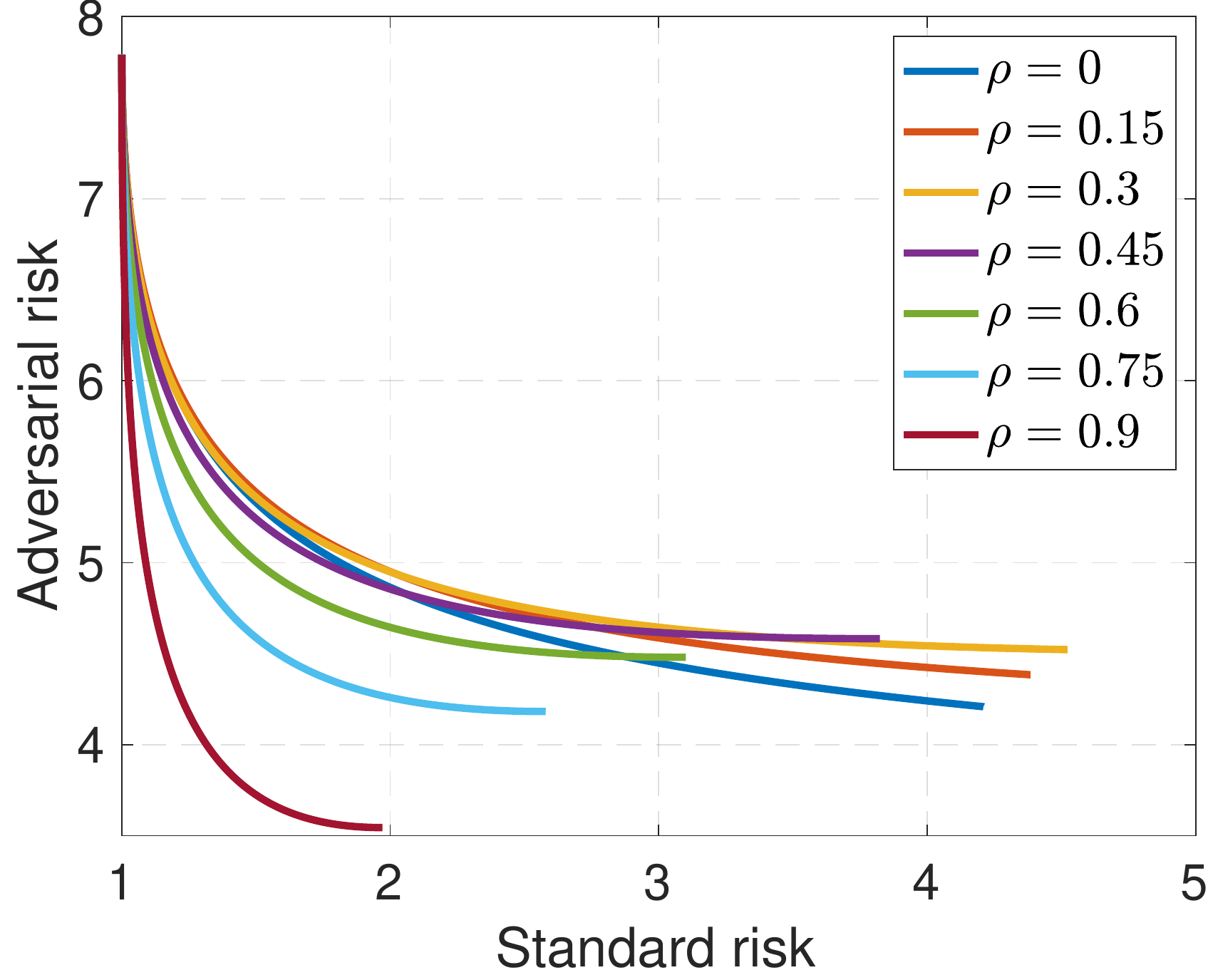}
	\caption{Pareto optimal curve for several feature dependency values $\rho$ with $d=10$ and $\eps=1$.}
	\label{fig:linear:rho}
	\end{subfigure}
\hfill
\begin{subfigure}[b]{0.32\textwidth}
\includegraphics[scale=0.31]{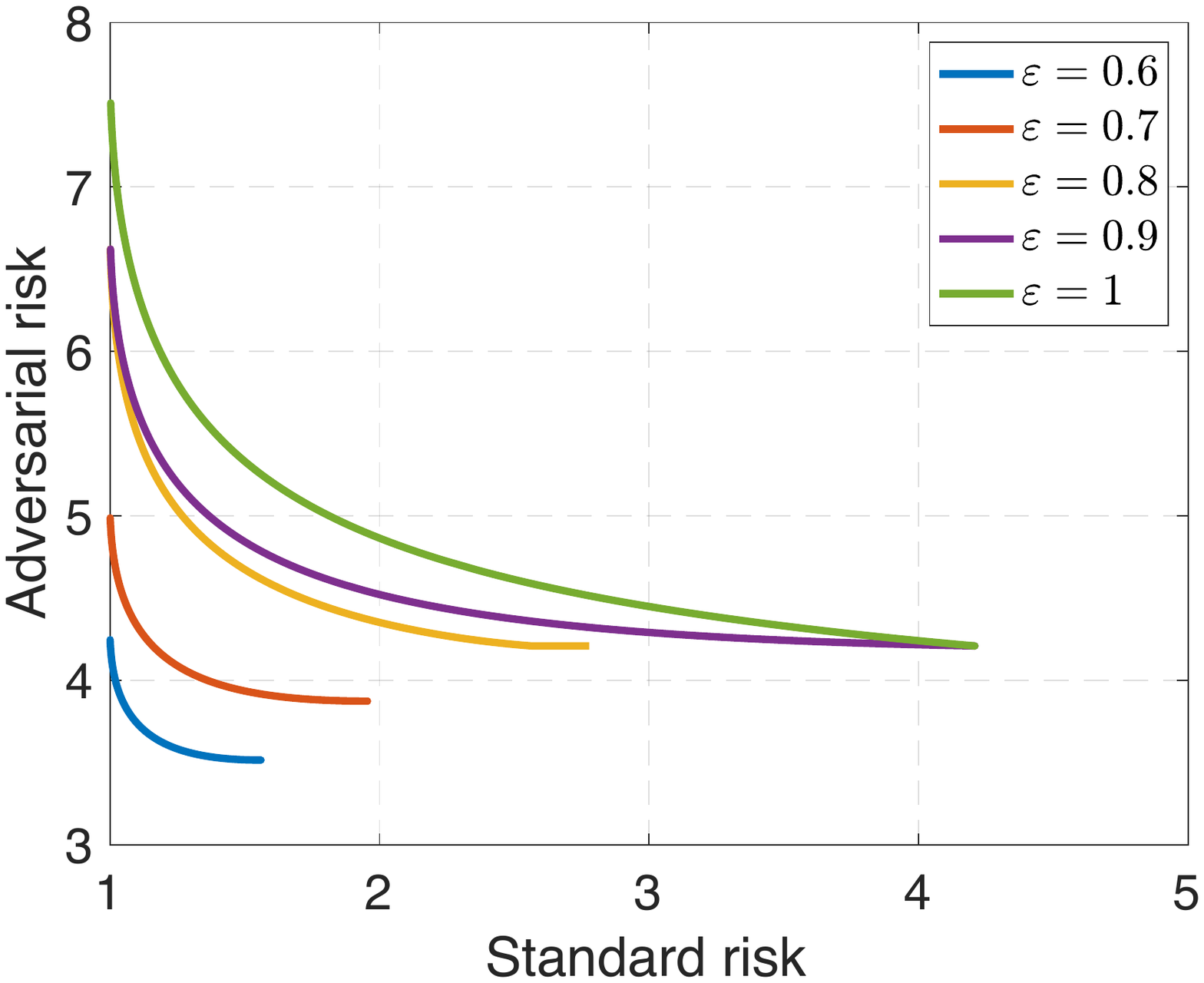}
\caption{Pareto optimal curve for several adversary's power  $\eps$ with $\rho=0$ and $d=10$.}	
\label{fig:linear:eps}	
\end{subfigure}
\caption{The effect of feature dimension ($d$), dependency across features ($\rho$), and adversary's power ($\eps$) on Pareto optimal tradeoff between adversarial ($\AR$) and standard risks ($\SR$) in linear regression setting.}
\label{fig:linear}
\end{figure}

\subsection{Binary classification}\label{sec:binary}
We next consider the problem of binary classification under a Gaussian mixture data model.
Under this model, each data point belongs to one of two two classes $\{\pm1\}$ with corresponding probabilities $\pi_+$, and $\pi_- = 1-\pi_+$. The feature vectors in each class are generated independently according to an isometric Gaussian distribution with mean $\{\pm \mu\}$ depending on the class. In other words, given label $y_i\in \{\pm1\}$, the feature vector $x_i\in\reals^d$ is drawn from $\normal(y_i\mu, \Sigma)$.  

We focus on class of linear classifiers $\{x^\sT\th: \th\in\reals^d\}$. Given a model $\th$ the predicted labels are simply given as $\sign(x^\sT \th)$. We consider 0-1 loss $\ell(\th;z) = \ind(\hat{y}\neq y) = \ind(y x^\sT\th\le0)$.
We also consider Wasserstein adversarial training with distance
\begin{align}\label{eq:distance2}
d(z,\tilde{z}) = \rnorm{x-\tilde{x}} + \infty\cdot \ind\{y\neq\tilde{y}\}\,,\quad z=(x,y), \quad \tilde{z} = (\tilde{x},\tilde{y})\,,
\end{align}

Our next results is on characterizing the standard risk and the Wasserstein adversarial risk for this model.
\begin{propo}\label{propo:binary-AR}
Consider binary classification with Gaussian mixture data model and 0-1 loss. Let $a_\th:= \frac{\mu^\sT\th}{\twonorm{\Sigma^{1/2}\th}}$. Then, for a linear classifier $x\mapsto \sgn(x^\sT\th)$, the standard risk is given by
\[
\SR(\th) = \Phi(-a_\th)\,,
\]
where $\Phi(z) = \frac{1}{\sqrt{2\pi}} \int_{-\infty}^z e^{-\frac{t^2}{2}}\de t$ denotes the c.d.f of a standard Gaussian distribution.

In addition, the Wasserstein adversarial risk with $p=2$ and metric $d$ given by \eqref{eq:distance2} can be characterized as follows
\begin{align}\label{eq:AR-Bin1}
\AR(\th) = \inf_{\gamma\ge0 } \bigg[ &\frac{\gamma}{b_\th}\eps^2+\Phi\Big( \sqrt{\frac{2}{\gamma}}- a_\th\Big) 
\nonumber \\
&+\frac{\gamma}{2}
\left\{\Big(a_\th +\sqrt{\frac{2}{\gamma}}\Big) \varphi\Big(a_\th - \sqrt{\frac{2}{\gamma}}\Big) - a_\th\varphi(a_\th)
+(a_\th^2+1) \Big[\Phi\Big(a_\th - \sqrt{\frac{2}{\gamma}}\Big)  -  \Phi(a_\th)  \Big]\right\}\bigg].
\end{align}
with $b_\th:=\frac{\twonorm{\Sigma^{1/2}\th}^2}{\qnorm{\th}^2}$, $\ell_q$ denoting the dual norm of $\ell_r$ (i.e., $\frac{1}{r}+\frac{1}{q}=1$), and $\varphi(t):=\frac{1}{\sqrt{2\pi}}e^{-\frac{t^2}{2}}$ standing for the  p.d.f of a standard Gaussian distribution.
\end{propo} 
Note that as an implication of Proposition~\ref{propo:binary-AR}, the standard risk $\SR(\th)$ and the adversarial risk $\AR(\th)$ depend on the estimator $\th$ only through the components $a_\th = \frac{\mu^\sT\th}{\twonorm{\th}}$ and $b_\th=\frac{\twonorm{\Sigma^{1/2}\th}^2}{\qnorm{\th}^2}$.

We next characterize the Pareto optimal front for the region $\{(\SR(\th),\AR(\th)): \th\in\reals^d\}$. 
Since the 0-1 loss $\ind(yx^\sT \th\le0)$ is convex in $\th$, both the standard risk and the adversarial risks are convex functions of $\th$ (by a similar argument given prior to Theorem~\ref{thm:pareto-linear}.)
 
\begin{thm}\label{thm:pareto-binary}
Consider the setting of Proposition~\ref{propo:binary-AR} and define the function $F(\theta,\gamma):\reals^{d+1}\mapsto \reals_{\ge0}$ given by
\begin{align*}
F(\theta,\gamma) = &\frac{\gamma}{b_\th}\eps^2+\Phi\Big( \sqrt{\frac{2}{\gamma}}- a_\th\Big) 
\nonumber \\
&+\frac{\gamma}{2}
\left\{\Big(a_\th +\sqrt{\frac{2}{\gamma}}\Big) \varphi\Big(a_\th - \sqrt{\frac{2}{\gamma}}\Big) - a_\th\varphi(a_\th)
+(a_\th^2+1) \Big[\Phi\Big(a_\th - \sqrt{\frac{2}{\gamma}}\Big)  -  \Phi(a_\th)  \Big]\right\}
\end{align*}
with $a_\th= \frac{\mu^\sT\th}{\twonorm{\Sigma^{1/2}\th}}$ and $b_\th:=\frac{\twonorm{\Sigma^{1/2}\th}^2}{\qnorm{\th}^2}$.
Consider the following minimization problem
\begin{align}\label{eq:binary-weighted}
(\th^\lambda_*,\gamma^\lambda_*) := \arg\min_{\gamma\ge0, \th} \lambda\Phi(-a_\th) + F(\th,\gamma)\,.
\end{align}
The Pareto optimal curve is given by $\{(\Phi(-a_{\th^\lambda_*}),F(\th^\lambda_*,\gamma^\lambda_*)): \lambda\ge 0\}$. 
\end{thm}
Theorem~\ref{thm:pareto-binary} follows from the fact that the Pareto front of a convex set is characterized by intersection points of the set with the supporting hyperplanes. 

\begin{remark}\label{rem:r2}
For $r=q=2$ and $\Sigma =I$, we have $b_\th = 1$. In this case the objective of \eqref{eq:binary-weighted} is decreasing in $a_\th$ and since $|a_\th| \le  \twonorm{\mu}$, it is minimized at  $a_\th=\twonorm{\mu}$. In addition, $\SR(\th)$ is decreasing in $a_\th$ and is minimized at the same value of $a_\th=\twonorm{\mu}$. Therefore, the Pareto-optimal curve shrinks to a single point given by
\begin{align}\label{eq:AR-Bin2}
\SR &= \Phi(-\twonorm{\mu})\,,\\
\AR &= \inf_{\gamma\ge0 } \bigg[ {\gamma}\eps^2+\Phi\Big( \sqrt{\frac{2}{\gamma}}-\twonorm{\mu}\Big) 
\nonumber \\
&\quad\quad\quad+\frac{\gamma}{2}
\bigg\{\Big(\twonorm{\mu} +\sqrt{\frac{2}{\gamma}}\Big) \varphi\Big(\twonorm{\mu} - \sqrt{\frac{2}{\gamma}}\Big)-\twonorm{\mu}\varphi(\twonorm{\mu})\nonumber\\
&\quad\quad \quad\quad\quad\;\;+(\twonorm{\mu}^2+1) \Big[\Phi\Big(\twonorm{\mu} - \sqrt{\frac{2}{\gamma}}\Big)  -  \Phi(\twonorm{\mu})  \Big]\bigg\}\bigg]. \nonumber
\end{align}
In other words, the tradeoff between standard and adversarial risks, achieved by linear classifiers, vanishes in this case and the estimators in direction of the class average $\mu$ are optimal with respect to both standard risk and the Wasserstein adversarial risks. 
\end{remark}
We refer to Section~\ref{proof:binary-r2} for the proof of Remark~\ref{rem:r2}.

Figure \ref{fig:binary} showcases the effect of different factors in a binary classification setting on the Pareto-optimal tradeoff between standard and adversarial risks. 
Here the features $x$ are drawn from $\normal(y\mu,\Sigma)$, with $\Sigma_{ij}=\rho^{|i-j|}$. The class average $\mu$ has i.i.d entries from $\normal(0,1/d)$ with $d=10$. In Figure \ref{fig:binary:q}, we investigate the role of the norm $r$ used in the Wasserstein adversary model, cf. Equation \eqref{eq:distance2}. As discussed in Remark~\ref{rem:r2}, when $r=2$, the tradeoff between standard and adversarial risks vanishes and the estimators in direction of the class average $\mu$ are optimal with respect to both standard risk and the Wasserstein adversarial risks.  

Figure \ref{fig:binary:rho} illustrates the effect of dependency among features $\rho$ on optimal tradeoff between standard and adversarial risks. In this setting $r=\infty$ and $\eps=0.3$. From the result of Theorem~\ref{thm:pareto-binary}, we see that these risks very much depend on the interaction between the class average $\mu$ and the features covariance $\Sigma$ and so the curves are shifted in highly nontrivial way depends on the value of $\rho$ when we fix $\mu$. 

The role of adversary's budget $\eps$ is depicted in figure \ref{fig:binary:eps} in which $r=\infty$, $\rho=0$. Similar to the linear regression setting, when $\eps$ is small the two measures of risk are close to each other and we have a small range of Pareto-optimal models. As $\eps$ grows, the standard risk and the adversarial risks differ significantly and we get a wide range of Pareto-optimal models.

\begin{figure}
	\centering
	\begin{subfigure}[b]{0.32\textwidth}
		\centering
		\includegraphics[scale=0.31]{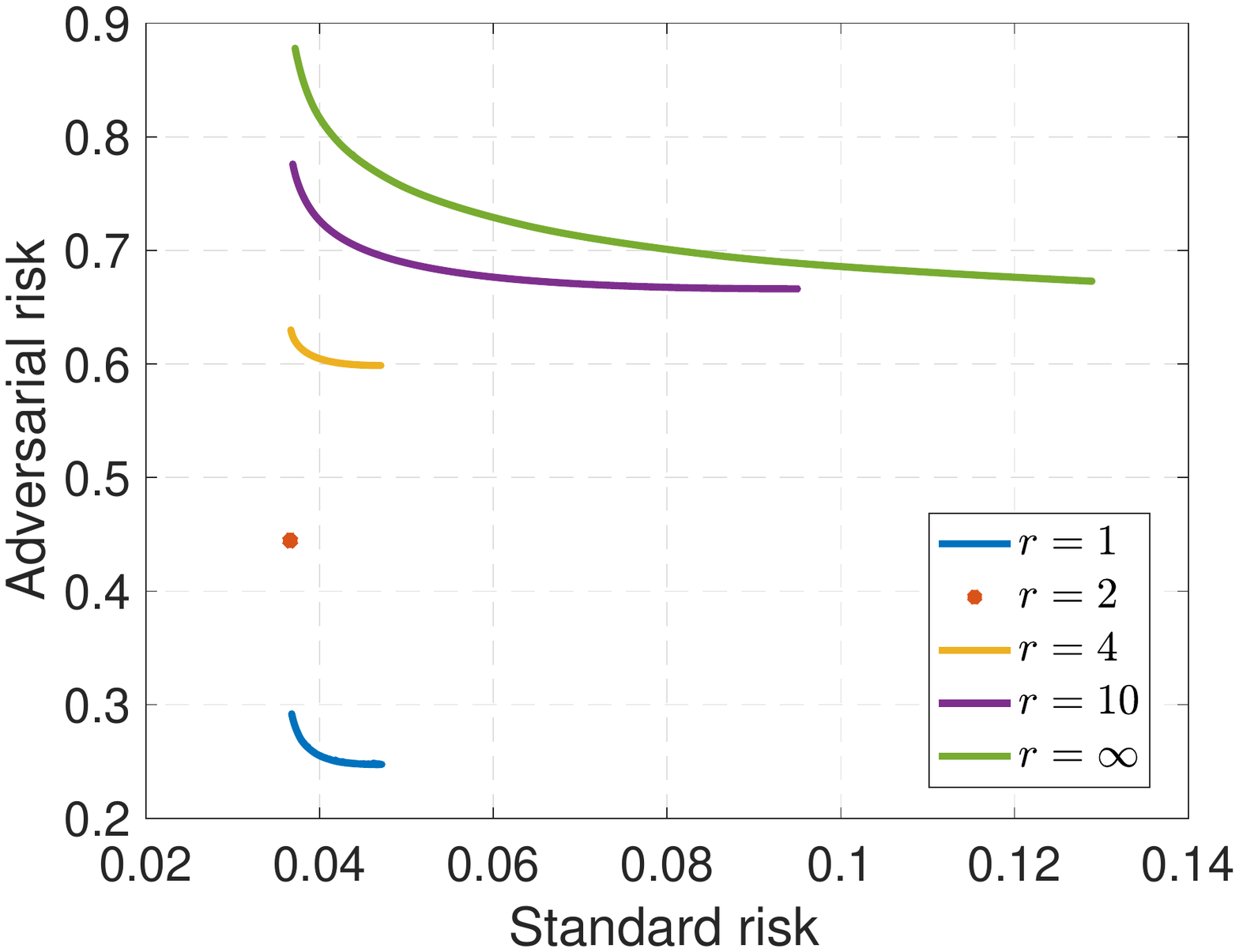}
		\caption{Pareto optimal curve for several $\ell_r$ norms on feature space with $d=10$, $\eps=0.5$ and $\rho=0$. }
		\label{fig:binary:q}
	\end{subfigure}
	\hfill
	\begin{subfigure}[b]{0.32\textwidth}
		\centering
		\includegraphics[scale=0.31]{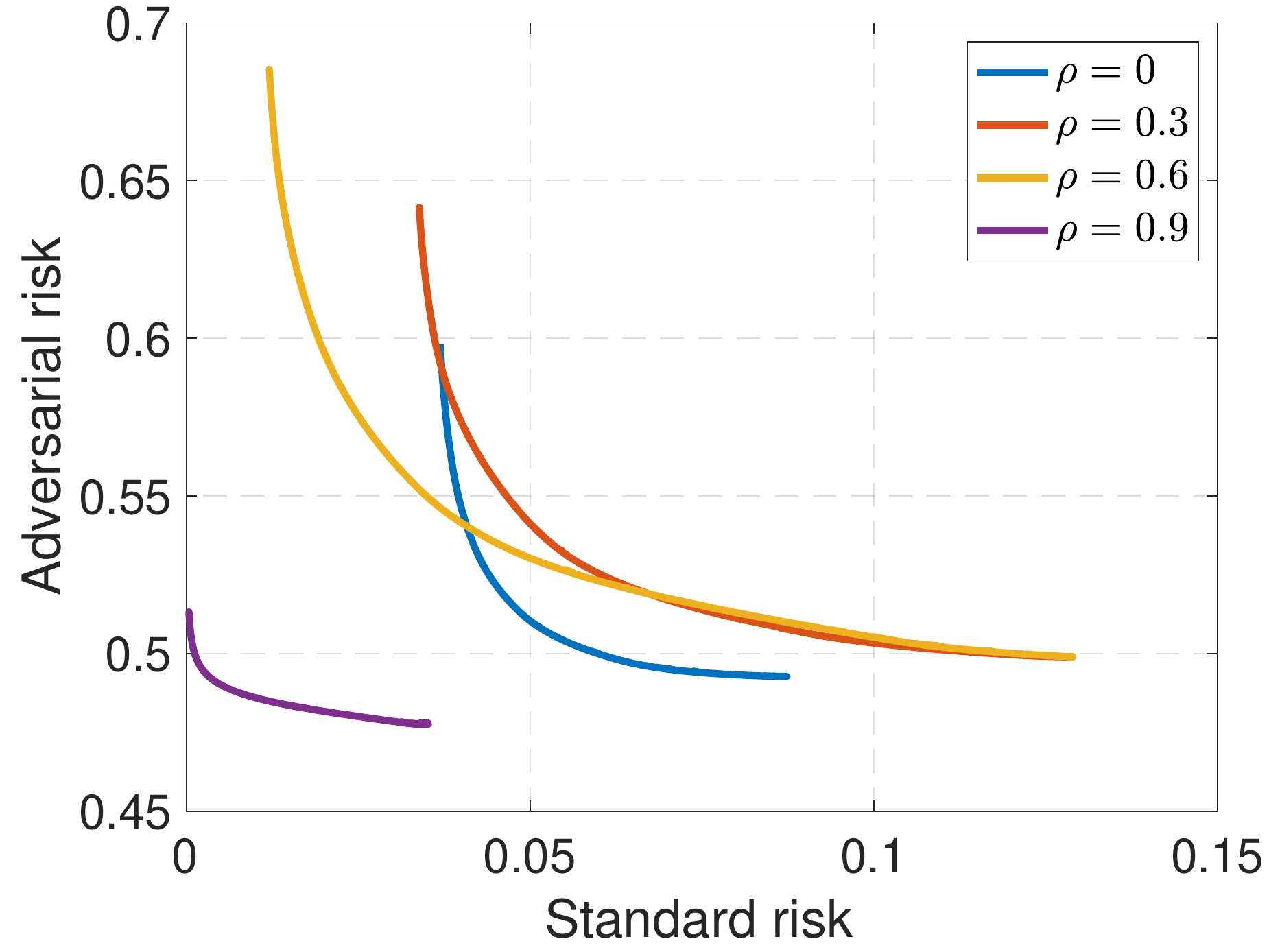}
		\caption{Pareto optimal curve for several feature dependency values $(\rho)$  with $d=10$, $\eps=0.3$, and $r=\infty$. }
		\label{fig:binary:rho}
	\end{subfigure}
	\hfill
	\begin{subfigure}[b]{0.32\textwidth}
		\includegraphics[scale=0.31]{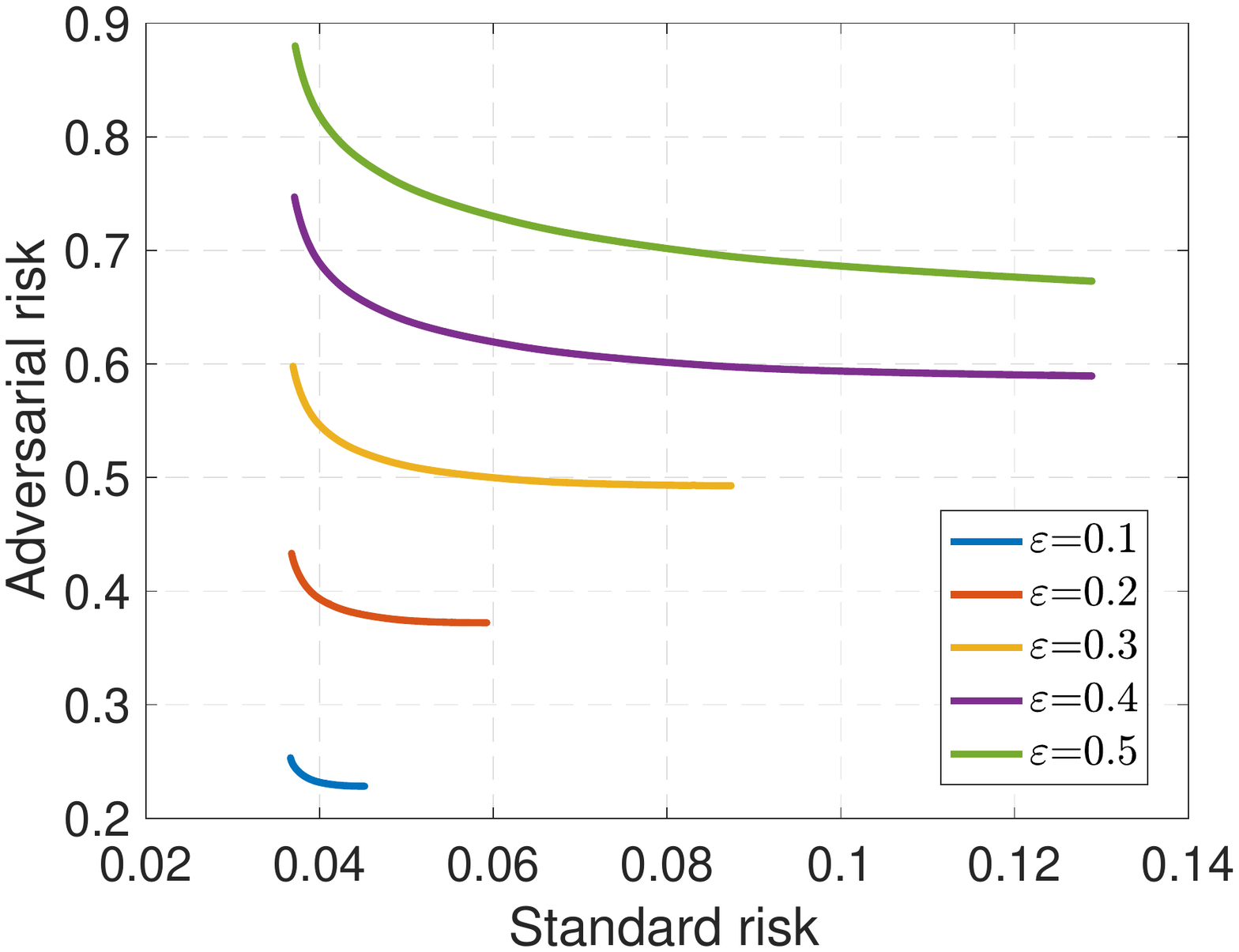}
		\caption{Pareto optimal curve for several adversary's budget $\eps$ with $d=10$, $r=\infty$,  and $\rho=0$.}	
		\label{fig:binary:eps}	
	\end{subfigure}
	\caption{The effect of defined $\ell_r$ norm on feature space, dependency across features ($\rho$) and adversary's power $\eps$ on Pareto optimal tradeoff between adversarial and standard risks in binary classification under Gaussian mixture model.  }
	\label{fig:binary}
\end{figure}

\subsection{Learning nonlinear functions}\label{sec:nonlinear}
We next investigate the tradeoff between standard and adversarial risk for the problem of learning  an unknown function over the $d$-dimensional sphere $\mathbb{S}^{d-1}$.
More precisely, we consider the following data generative model:
\begin{align}\label{eq:f-NL}
y = f_d(x) + w\,,
\end{align}
with $x\sim{{\sf Unif}}(\mathbb{S}^{d-1}(\sqrt{d}))$, the $d$-dimensional sphere of radius $\sqrt{d}$, and $w\sim\normal(0,\sigma^2)$ independent of $x$.
We consider fitting a random features model to data generated according to \eqref{eq:f-NL}. The class of random features model is given by
\begin{align}\label{eq:RF}
\cF_{{\rm RF}}(\theta,U) = \left\{f(x,\theta,U): = \sum_{i=1}^N \theta_i \sigma(u_i^\sT x):\quad \theta_i\in\reals,\; i=1,\dotsc, N \right\}\,,
\end{align}
where $U\in\reals^{N\times d}$ is a matrix whose $i$-th row is the vector $u_i$, uniformly drawn from $\mathbb{S}^{d-1}(1)$, independently from data. 
The random features model can be equivalently represented by two-layer neural network with the first-layer weights $U$ chosen randomly and $\theta =(\theta_i)_{1\le i\le N}$ corresponding to the 
second-layer weights. The random features model was introduced by~\cite{rahimi2007random} for scaling kernel methods to large datasets. There is indeed a substantial literature drawing connections between random features models, kernel methods and fully trained neural networks~\cite{daniely2016toward,daniely2017sgd,jacot2018neural,li2018learning}. 
In~\cite{mei2019generalization}, the generalization error (standard risk) of random features model was precisely characterized for the problem of learning a function $f_d(\cdot)$ over $\mathbb{S}^{d-1}(\sqrt{d})$ in the regime where the network width $N$, sample size $n$ and feature dimension $d$ grow in proportion. The nonlinear model considered in~\cite{mei2019generalization} is of the form
\begin{align}
f_d(x) = \beta_{d,0} + x^\sT \beta_{d,1} +  f_d^{{\rm NL}}(x)\,,
\end{align}  
with the nonlinear component $f_d^{{\rm NL}}(x)$ is a centered isotropic Gaussian process indexed by $x$.  We follow the same model and consider the following random quadratic function
\begin{align}\label{eq:NL-quad}
f_d(x) = \beta_{d,0} + x^\sT \beta_{d,1} + \frac{F_*}{d} [x^\sT Gx - \Tr(G)]\,,
\end{align}
for some fixed $F_*\in\reals$ and $G\in\reals^{d\times d}$ a random matrix with i.i.d entries from $\normal(0,1)$.

Our goal is to study the Pareto-optimal tradeoff between standard and adversarial risks for this learning setting, achieved by the class of random features model~\eqref{eq:RF}.
The standard risk in this setting is given by
\begin{align}\label{eq:SR-NL}
\SR(\th) = \E_{x,y}\left[(y-\theta^\sT \sigma(Ux))^2\right] = \E_{x}\left[  (f_d(x)- \theta^\sT \sigma(Ux) ) ^2\right] +\sigma^2\,.
\end{align}
For the Wasserstein adversarial risk we use the following corollary which is obtained by specializing Proposition \ref{pro:approx} to random features model.

%
%We consider linear regression on $N$ random features  of the form $\sigma(w_a^T x)$, for $a = 1,\dotsc, N$ and $\sigma$ the ReLU activation function. 
%
%the consider the class of two-layer neural network models with random features. We assume $u = \sigma(Wx)$ with $x\in\reals^d$ has i.i.d entries from $\normal(0,1)$ and $W\in \reals^{k\times d}$ is the weights of randomized feature mapping in the first layer with i.i.d entries from $\normal(0,1/d)$. In addition, $\sigma$ is an activation function acting componentwise. Therefore,  the fitted model with square loss takes the following standard risk:
%
%\begin{equation}\label{eq:SR-nonlinear}
%\SR(\th,W)=\E\left[  (y- \theta^\sT \sigma(Wx) ) ^2\right]\,,
%\end{equation}
% where $\th$ corresponds to the weights of the second layer.
%This is the random feature model introduced by~\cite{??} for scaling kernel methods to large datasets. In the next propostion, by using the result from proposition \ref{pro:approx}  we characterzes the first-order approximation for adversarial risk in two-layer network models. 
\begin{coro}\label{coro:AR-nonlinear}
Consider the class of feature model given by~\eqref{eq:RF}. In this case, the 2-Wasserstein adversarial risk with distance $d(\cdot,\cdot)$~\eqref{eq:distance} admits the following first-order approximation:
\[
	\AR(\th)= \SR(\th)+ 2\eps\; \E_x\left[\big[(f_d(x)- \theta^\sT \sigma(Ux) ) ^2 +\sigma^2\big] \twonorm{U^\sT \diag(\sigma'(Ux)) \th}^2 \right]^{1/2} +O(\eps^2)\,,
\] 
with $\sigma'(\cdot)$ denoting the derivative of the activation $\sigma(\cdot)$ and $\SR(\th)$ given by~\eqref{eq:SR-NL}.
\end{coro}	
The proof of Corollary \ref{coro:AR-nonlinear} is given in Appendix~\ref{proof:coro:AR-nonlinear}.

The standard risk is quadratic and hence convex in $\th$. The adversarial risk is also convex in $\theta$ (it follows from the fact that pointwise maximization preserves convexity.) Therefore, for small values of $\eps$ (weak adversary) the first order approximation of $\AR(\th)$ is also convex in $\theta$. As such, (almost) all Pareto optimal points are given by minimizing a weighted combination of the two risk measures as the weight $\lambda$ varies in $[0,\infty)$:
\begin{align}\label{eq:weightedSumNL}
\theta_\lambda &:=\arg\min_{\theta}\; \lambda \SR(\theta) + \AR(\theta)\nonumber\\
&=\arg\min_{\theta}\;\; (1+\lambda) \SR(\theta) + 2\eps\; \E_x\left[\big[(f_d(x)- \theta^\sT \sigma(Ux) ) ^2 +\sigma^2\big] \twonorm{U^\sT \diag(\sigma'(Ux)) \th}^2 \right]^{1/2}\,.
\end{align}

We use the above characterization to derive the Pareto-optimal tradeoff curves between standard and adversarial risks for learning function $f_d(x)$, given by \eqref{eq:NL-quad}
with $F_*=1$, $\beta_{d,0} = 0$, and $\beta_{d,1}\in \reals^d$ with i.i.d entries $\sim\normal(0,1/d)$. The data are generated according to~\eqref{eq:f-NL} with $\sigma =2$, $d = 10$ and $N\in\{250,500,750,1000\}$. To compute $\theta_\lambda$ we use empirical loss with $n = 500K$ samples of $x\sim{{\sf Unif}}(\mathbb{S}^{d-1}(\sqrt{d}))$. For each value of $\lambda$ and $N$ we generate $15$ realization of weights $U$ and compute $\theta_\lambda$ for each realizations. The Pareto optimal points $\{\SR(\th_\lambda),\AR(\th_\lambda): \, \lambda\ge0\}$ are plotted in Figure~\ref{fig:NL}. As we see  for each value of $N$ the tradeoff curves concentrate as $N$ grows implying that the estimator $\theta_\lambda$ becomes independent of the specific realization of weights $U$. Also we observe that the tradeoff between standard and adversarial risks exist even for large values of $N$. Interestingly, as the network width $N$ grows both the standard risk and adversarial risk decrease but the tradeoff between them clearly remains (the length of Pareto front does not shrink).
 
\begin{figure}
	\centering
\includegraphics[scale=0.5]{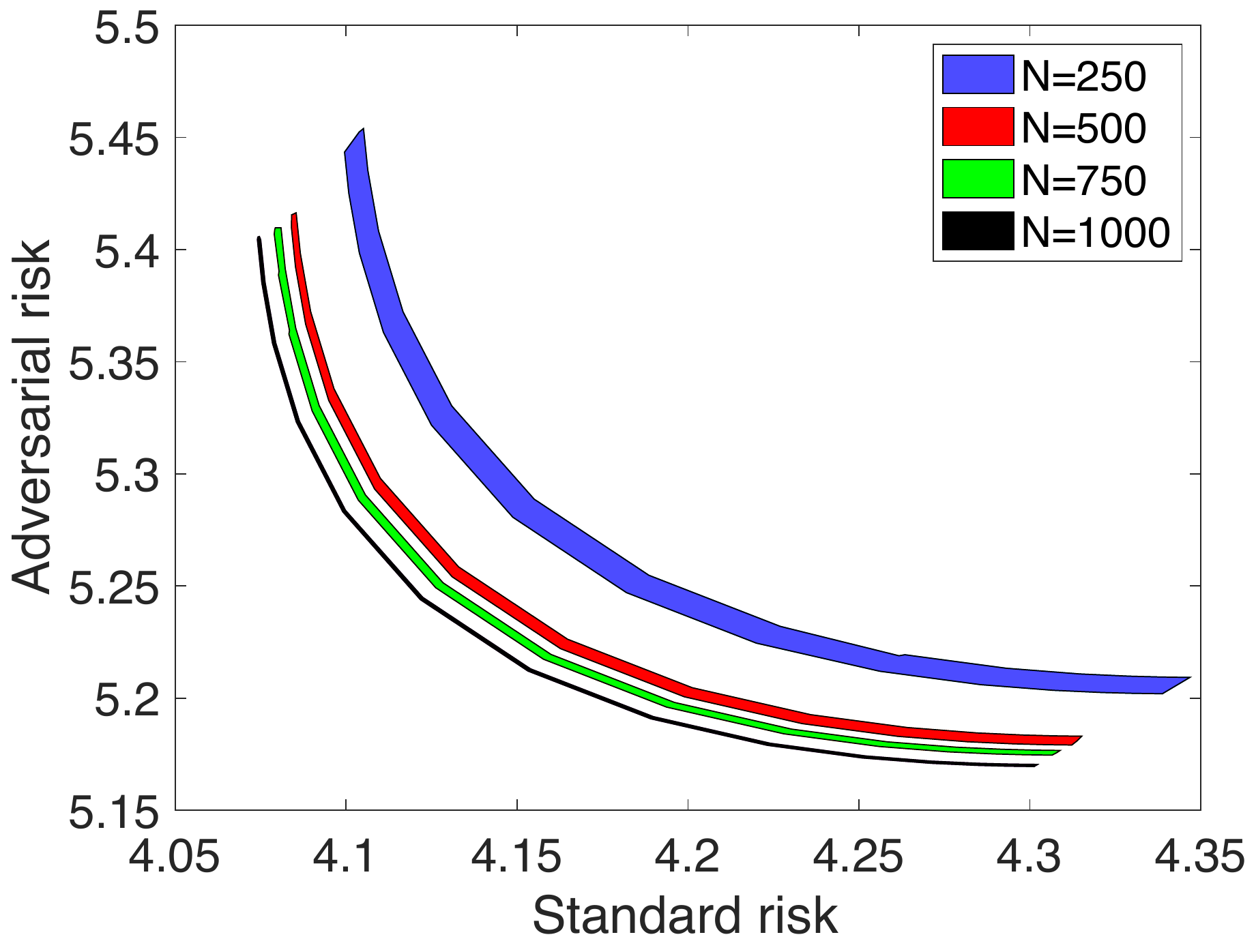}
\caption{Pareto-optimal tradeoff curves for learning random quadratic functions using random features model. Data is generated according to \eqref{eq:f-NL} with $\sigma =2$ and $f_d(x)$ given by \eqref{eq:NL-quad}. Here, $d=10$ and $N$ is the number of random features (width of the neural network).}
\label{fig:NL}
\end{figure} 

\bibliographystyle{alpha}
\bibliography{mybib,Bibfiles2}

%=================
\section{Proof of theorems and technical lemmas}
\subsection{Proof of Proposition~\ref{propo:AR-linear}}\label{proof:AR-linear}
%\aj{To be updated}\\
%\mm{done}
From the definition of robust surrogate in~\eqref{eq:phi} for the setting of Proposition \ref{propo:AR-linear} we have 
 $$\phi_{\gamma}(\th;z_0):=\sup\limits_{x}^{}  \left\{(y_0-x^\sT\th)^2-\gamma {\twonorm{x-x_0}^2} \right\} \,,$$
 by introducing $g_{\gamma}(x):= (y_0-x^\sT\th)^2-\gamma {\twonorm{x-x_0}^2}$,  for every scalar $c$ we get
\begin{align*}
    &g_{\gamma}(x_0+c\th)= g_{\gamma}(x_0)+ 2c(x_0^\sT\th-y_0)\twonorm{\th}^2 +c^2\twonorm{\th}^2(\twonorm{\th}^2-\gamma)\,,
    \end{align*}
this implies if $\gamma < \twonorm{\th}^2$, then $\phi_{\gamma}(\th;z_0)=+\infty$. Consider $\gamma\geq \twonorm{\th}^2$, then from relation $\nabla^2g_{\gamma}(x)=2(\th\th^\sT-\gamma I)$ we realize that $g_\gamma$ is concave. First order optimal conditions imply 

\[
    x^*=x_0+\frac{x_0^\sT\th-y_0}{\gamma-\twonorm{\th}^2}\th\,.
\]
Replacing $x^* $ in $g_{\gamma}$ yields
\begin{equation}
\phi_{\gamma}(\th;z)=
\begin{cases}
 +\infty &\text{if } \gamma < \twonorm{\th}^2\,,\\
\frac{\gamma(y_0-x^\sT\th)^2}{(\gamma-\twonorm{\th}^2)}&\text{if }\gamma \geq \twonorm{\th}^2 \,. \\ 
\end{cases}
\end{equation}
Then, we use dual formulation~\eqref{eq:dual} to compute Wasserstein adversarial risk:
\begin{align*}
   &\AR(\th):= \sup_{\Q\in \mathcal{U}_\eps(\prob_Z)}\; \E_{z\sim \Q} \left[\ell(\th;z) \right]\\
  &=\inf\limits_{\gamma \geq 0}^{} \{\gamma\eps^2 + \E_{\prob_Z}[\phi_{\gamma}(\th;z)] \} \\
&=\inf\limits_{\gamma \geq \twonorm{\th}^2}^{} \{\gamma\eps^2 + \E_{\prob_Z}[\phi_{\gamma}(\th;z)] \}\\
&=\inf\limits_{\gamma \geq \twonorm{\th}^2}^{} \{\gamma\eps^2 + \frac{\gamma \E_{\prob_Z}[\ell(\th;z)]}{\gamma-\twonorm{\th}^2} \}\,,
\end{align*}
the infimum is achieved at $\gamma^*=\frac{1}{\eps}\sqrt{\E_{\prob_Z}[\ell(\th;z)]}\twonorm{\th}+\twonorm{\th}^2$. Finally, this gives us

\[\AR(\th)=\left(\sqrt{\E_{\prob_Z}[\ell(\th;z)]} + \eps\twonorm{\th}\right)^2 \,.\]

%=================================

\subsection{Proof of Theorem~\ref{thm:pareto-linear} }\label{proof:thm:pareto-linear}
Define $\cR(\th):=\lambda\SR(\th)+\AR(\th)$. Proposition~\ref{propo:AR-linear} implies
$\AR(\th)=\SR(\th)+2\eps\twonorm{\th}\sqrt{\SR(\th)}+ \eps^2\twonorm{\th}^2$, then by expanding adversarial risk relation $\AR(\th)$ in $\cR(\th)$ we get 

\begin{equation}\label{eq:proof-linear-risk}
 \cR(\th)=  (1+\lambda)\SR(\th)+\eps^2\twonorm{\th}^2+2\eps\twonorm{\th}\sqrt{\SR(\th)}  \,.
\end{equation}
It is easy to see $\SR(\th)=\sigma^2_y+\th^\sT\Sigma\th-2v^\sT\th$. Replace $\nabla_{\th}\SR(\th)=2(\Sigma\th-v)$ in~\eqref{eq:proof-linear-risk} to get  
\begin{equation}\label{eq:linear-grad-cR}
	\nabla_{\th}\cR(\th)=2(1+\lambda)(\Sigma\th-v)+2\eps^2\th+2\eps\left( \frac{\th}{\twonorm{\th}}\sqrt{\SR(\th)} +(\Sigma\th-v) \frac{\twonorm{\th}}{\sqrt{\SR(\th)} }  \right)\,,
\end{equation}
therefore stationary points (solutions of $\nabla_\th\cR(\th)=0$) and a critical point $\th=0$ are candidates for global minimizers. From equation $\SR(\th)=\sigma^2_y+\th^\sT\Sigma\th-2v^\sT\th$ and adversarial risk relation in Proposition~\ref{propo:AR-linear} it is clear that for $\th=0$ we have $\SR(\th)=\AR(\th)=\sigma^2_y$.  Next, we focus on characterizing stationary minimizers of $\cR(\th)$ and their corresponding standard and adversarial risk values. If $\th_*$ is a stationary point, then putting~\eqref{eq:linear-grad-cR} to be zero yields
\begin{equation}\label{eq:proof-linear-stationary}
\left(\left(1+\lambda+\frac{\eps\twonorm{\th_*}}{\sqrt{\SR(\th_*)}}\right)\Sigma+\left(\eps^2+\frac{\eps\sqrt{\SR(\th_*)}}{\twonorm{\th_*}} \right) I  \right )\th_*=\left(1+\lambda+\frac{\eps\twonorm{\th_*}}{\sqrt{\SR(\th_*)}}\right) v\,.
\end{equation}

Introduce $A_*:=  \frac{\sqrt{\SR(\th_*)}}{\twonorm{\th_*}}$ and $\gamma_*:=\frac{\eps^2+\eps A_*}{1+\lambda+\frac{\eps}{A_*}}$, then~\eqref{eq:proof-linear-stationary} can be simplified to
$\th_*=(\Sigma+\gamma_* I)^{-1}v$. By replacing $\th_*=(\Sigma+\gamma_* I)^{-1}v$ in $A_*$ along with equation
$\SR(\th)=\sigma^2_y+\th^\sT\Sigma\th-2v^\sT\th$ we get
\begin{align*}
	&A_*=\frac{\sqrt{\SR((\Sigma+\gamma_* I)^{-1}v)}}{\twonorm{(\Sigma+\gamma_* I)^{-1}v}}\\
	&=\frac{1}{\twonorm{(\Sigma+\gamma_* I)^{-1}v}} \left(\sigma_y^2 + \twonorm{\Sigma^{1/2}(\Sigma+\gamma_* I)^{-1}v}^2 - 2v^\sT(\Sigma+\gamma_* I)^{-1} v\right)^{1/2}\,,
\end{align*}
therefore $\gamma_*$ is a fixed point solution of two equations~\eqref{eq:linear-fixedpoints1} and~\eqref{eq:linear-fixedpoints2}. Moreover, definition of $A_*$ gives us $\SR(\th_*)=A_*^2\twonorm{(\Sigma+\gamma_* I)^{-1}v}^2$. Next, from adversarial risk relation in Proposition~\ref{proof:AR-linear} we know that
$\AR(\th_*)=(\sqrt{\SR(\th_*)}+\eps\twonorm{\th_*} )^2$. This implies $\AR(\th_*)=(A_*+\eps)^2\twonorm{(\Sigma+\gamma_* I)^{-1}v}^2$. 

%=================================
\subsection{Proof of Corollary~\ref{coro:pareto-linear}}\label{proof:coro-linear}
For linear data model $y=x^\sT\th_0+w$ with isotropic features $\E[xx^T]=I_d$ and Gaussian noise $w\sim\normal(0,\sigma^2)$ we have $\E[xy]=\th_0$. In addition, we have $\E[y^2]=\sigma^2+\twonorm{\th_0}^2$. This gives us  $\sigma^2_y=\sigma^2+\twonorm{\th_o}^2$. Use Theorem \ref{thm:pareto-linear} with $v=\th_0$, $\Sigma=I$, and $\sigma_y^2=\sigma^2+\twonorm{\th_0}^2$ to get Corollary \ref{coro:pareto-linear}.

%=================================
\subsection{Proof of Proposition~\ref{propo:binary-AR}}\label{proof:pareto-binary}
We start by proving the expression for standard risk. By definition we have
\begin{align}
\SR(\th) &:= \E[\ind(y\neq\hat{y})] = \prob(yx^\sT\th \le0)\nonumber\\
&=\prob\left(y(y\mu+\Sigma^{1/2} u)^\sT\th \le 0\right)\nonumber\\
&=\prob\left((\mu+\Sigma^{1/2} u)^\sT\th \le 0\right)\nonumber\\
&= \prob\left(\mu^\sT\th + \twonorm{\Sigma^{1/2} \theta} \nu \le 0\right)\nonumber\\
&= \Phi\left(-\frac{\mu^\sT\th}{\twonorm{\Sigma^{1/2}\th}}\right)\,,
\end{align}
with $u\sim\normal(0,I_d)$ and $\nu\sim\normal(0,1)$.
To prove the expression for adversarial risk we use the dual form~\eqref{eq:dual}.
Our next lemma characterizes the function $\phi_\gamma$ given by~\eqref{eq:phi} for the binary problem under the Gaussian mixture model. 
\begin{lemma}\label{lem:phi}
Consider the binary classification problem under the Gaussian mixture model with 0-1 loss. Then, the robust surrogate for 
the loss function $\phi_\gamma$ given by~\eqref{eq:phi} with distance $d(\cdot,\cdot)$~\eqref{eq:distance} satisfies
\begin{align*}
\E_{\prob_Z}[\phi_\gamma(\th;z)] = &\Phi\Big( \sqrt{\frac{2}{b_\th\gamma}}- a\Big) +\frac{b_\th\gamma}{2}
\Big\{\Big(a_\th + \sqrt{\frac{2}{b_\th\gamma}}\Big) \varphi\Big(a_\th - \sqrt{\frac{2}{b_\th\gamma}}\Big) - a_\th\varphi(a_\th)\\
&+(a_\th^2+1) \Big[\Phi\Big(a_\th - \sqrt{\frac{2}{b_\th\gamma}}\Big)  -  \Phi(a_\th)  \Big]\Big\}\,,
\end{align*}
with $a_\th=\frac{\mu^\sT\th}{\twonorm{\Sigma^{1/2} \th}}$ and $b_\th = \frac{\twonorm{\Sigma^{1/2}\th}^2}{\qnorm{\th}^2}$.
\end{lemma}
\begin{proof}[Proof (Lemma~\ref{lem:phi})]
By definition of the $\phi_\gamma$ function, for the setting of Lemma~\ref{lem:phi} we have
\[
\phi_\gamma(\th;z_0) = \sup_{x} \{\ind(y_0x^\sT\th \le0) - \frac{\gamma}{2} \rnorm{x-x_0}^2\}\,.
\]
We let $v_0:= y_0x_0$ and $v = y_0 x$. Given that $y_0\in\{\pm1\}$, the function $\phi_\gamma$ can be written as
\[
\phi_\gamma(\th;z_0) = \sup_{v} \{\ind(v^\sT\th \le0) - \frac{\gamma}{2} \rnorm{v-v_0}^2\}\,.
\]
First observe that by choosing $x = x_0$, we obtain $\phi_\gamma(\th,z_0)\ge 0$.
It is also clear that $\phi_\gamma(\th,z_0)\le 1$.
We consider two cases.

{\bf Case 1:} ($v_0^\sT\th \le 0$). By choosing $v= v_0$ we obtain that $\phi_\gamma(\th;z_0) \ge 1$ and hence $\phi_\gamma(\th;z_0) = 1$.

{\bf Case 2:}($v_0^\sT\th > 0$). Let $v_*$ be the maximizer in definition of $\phi_\gamma(\th;z_0)$. If $v_*^\sT \th >0$, then we have
\[
\phi_\gamma(\th;z_0) = \ind(v_*^\sT\th \le0) - \frac{\gamma}{2} \rnorm{v_*-v_0}^2 = - \frac{\gamma}{2} \rnorm{v_*-v_0}^2 \le 0\,.
\]
Therefore, $\phi_\gamma(\th;z_0) = 0$ in this case. We next focus on the case that $v_*^\sT \th \le0$. It is easy to see that in this case, $v_*$ is the solution of the following optimization:
\begin{align}
&\min_{v\in\reals^d} \quad \rnorm{v-v_0}\nonumber\\
&\text{subject to} \quad v^\sT \th \le0
\end{align} 
Given that $v_0^\sT \th>0$ by assumption, using the Holder inequality it is straightforward to see that the optimal value is given by $\rnorm{v-v_0} = \frac{v_0^\sT\th}{\qnorm{\th}}$, with $\frac{1}{r}+\frac{1}{q}=1$.

%Given that $v_0^\sT \th>0$ by assumption the solution to the above optimization is given by $v_* = \pproj_\th(v_0)$, the projection
%onto the orthogonal space of $\th$. 
The function $\phi_\gamma$ is then given by $\phi_\gamma(\th;z_0) = 1-\frac{\gamma}{2}
\Big(\frac{v_0^\sT\th}{\qnorm{\th}}\Big)^2.$ Putting the two conditions $v_*^\sT \th\le0$ and $v_0^\sT \th>0$ together, we obtain
\[
\phi_\gamma(\th;z_0) = \max\Big\{1-\frac{\gamma}{2}
\Big(\frac{v_0^\sT\th}{\qnorm{\th}}\Big)^2,0\Big\}\,,
\]
in this case.

Combining case 1 and case 2 we arrive at
 \begin{align}\label{eq:phi1}
 \phi_\gamma(\th;z_0) = \ind(v_0^\sT\th \le 0) +  \max\left(1-\frac{\gamma}{2}
\Big(\frac{v_0^\sT\th}{\qnorm{\th}}\Big)^2,0\right) \ind(v_0^\sT\th > 0)\,.
 \end{align}
For $(x_0,y_0)$ generated according to the Gaussian mixture model, we have $v_0^\sT \th=y_0 x_0^\sT \th = \mu^\sT\th + \twonorm{\Sigma^{1/2} \th}\nu$ with $\nu\sim\normal(0,1)$. Hence, 
\[
\Big|\frac{v_0^\sT\th}{\qnorm{\th}}\Big| = \bigg|\frac{\mu^\sT\th}{\qnorm{\th}}+ \frac{\twonorm{\Sigma^{1/2} \theta}}{\qnorm{\th}}\nu\bigg|.
\]
Letting $a_\th: = \frac{\mu^\sT\th}{\twonorm{\Sigma^{1/2}\th}}$, \eqref{eq:phi1} can be written as
 \begin{align}\label{eq:phi2}
 \phi_\gamma(\th;z_0) &= \ind(\nu \le -a_\th) +  \max\left(1-\frac{\gamma}{2}\frac{\twonorm{\Sigma^{1/2}\th}^2}{\qnorm{\th}^2}
(\nu+a_\th)^2,0\right) \ind(\nu > -a_\th)\nonumber\\
&= \ind(\nu \le -a_\th) +  \left(1-\frac{b_\th\gamma}{2}
(\nu+a_\th)^2\right) \ind\left(\sqrt{\frac{2}{b_\th\gamma}}-a_\th>\nu > -a_\th\right)\,,
 \end{align}
where $b_\th:= \frac{\twonorm{\Sigma^{1/2}\th}^2}{\qnorm{\th}^2}$. By simple algebraic calculation, we get
 \begin{align}
\E_{\prob_Z}[\phi_\gamma(\th;z)] =\; & \Phi\Big( \sqrt{\frac{2}{b_\th\gamma}}- a_\th\Big) +\frac{b_\th\gamma}{2}
\Big\{\Big(a_\th +\sqrt{\frac{2}{b_\th\gamma}}\Big) \varphi\Big(a_\th - \sqrt{\frac{2}{b_\th\gamma}}\Big) - a_\th\varphi(a_\th)\nonumber\\
&+(a_\th^2+1) \Big[\Phi\Big(a_\th - \sqrt{\frac{2}{b_\th\gamma}}\Big)  -  \Phi(a_\th)  \Big]\Big\}\,.
 \end{align}
\end{proof}
The claim of Proposition~\ref{propo:binary-AR} follows readily from Lemma~\ref{lem:phi} and the fact that strong duality holds for the dual problem~\eqref{eq:dual}, where we use the change of variable $\gamma \mapsto \frac{\gamma}{b_\th}$.

%=====================================
\subsection{Proof of  Remark~\ref{rem:r2}}\label{proof:binary-r2}
Recall the objective~\eqref{eq:binary-weighted} and define
 \begin{align*}
 \cR(a):=&\lambda\Phi(-a)+\gamma\eps^2+ \Phi\left(\sqrt{\frac{2}{\gamma}}-a\right)
 \nonumber\\
& +\frac{\gamma}{2}\left\{ (a+\sqrt{\frac{2}{\gamma}})\varphi\left(a-\sqrt{\frac{2}{\gamma}}\right)-a\varphi(a)+(a^2+1)\Bigg(\Phi\left(a-\sqrt{\frac{2}{\gamma}}\right)-\Phi(a)\Bigg)  \right\}.
\end{align*}
Then, we get $\frac{d\cR(a)}{da}=-\lambda\varphi(-a)+\gamma\left\{ \varphi\left(\sqrt{\frac{2}{\gamma}}-a\right)-\varphi(a)+ a \left(   \Phi\left(\sqrt{\frac{2}{\gamma}}-a\right)-\Phi(a) \right) \right\}$. Note that
 $$\frac{\partial}{\partial t}\Big( \varphi(t-a)-\varphi(a)+a\left( \Phi(t-a)-\Phi(a) \right)\Big)=\varphi(t-a)(2a-t)\,,$$ 
 and therefore the maximum of $ \varphi(t-a)-\varphi(a)+a\left( \Phi(t-a)-\Phi(a) \right)$ is achieved at $t=2a$. As a result $\frac{d\cR(a)}{da} \leq -\lambda\varphi(-a)<0$, which implies that the objective \eqref{eq:binary-weighted} is decreasing in $a$. Since $|a|\leq \twonorm{\mu}$, its infimum is achieved at $a=\twonorm{\mu}$.
  
 Equations~\eqref{eq:AR-Bin2} follows from~\eqref{eq:AR-Bin1} by substituting for $a_\th = \twonorm{\mu}$ and $b_\th =1$.

\subsection{Proof of Corollary \ref{coro:AR-nonlinear}}\label{proof:coro:AR-nonlinear}
Recall the distance $d(\cdot,\cdot)$ on the space  $\mathcal{Z} = \{z= (x,y),\, x\in\reals^d, y\in\reals\}$ given by $d(z,\tilde{z})=||x-\tx||_2+\infty\cdot\ind(y-\ty)$. This metric is induced from norm $\|z\| = \twonorm{x}+\infty\cdot\ind(y=0)$ with corresponding conjugate norm $\|z\|_* = \twonorm{x}$.  We will use Proposition \ref{pro:approx} to find the variation of loss $\ell$ and derive the first-order approximation for the Wasserstein adversarial risk. 
%Assume there is a constant $M$ such that for every $z\in \mathcal{Z}$, we have $\twonorm{z}\leq M$ and we only focus on models with $\twonorm{\th}\leq M$. 
Denoting by $u_j \in \reals^d$ be the $j$th row of matrix $U$, for $j=1,2,...,N$, we have 
\begin{align}\label{eq:nonlinear:nabla_ell}
\nabla_x \ell(\th;Z)&=\nabla_x (y-\th^\sT\sigma(Ux))^2\nonumber\\
&=2(\th^\sT\sigma(Ux)-y)\sum\limits_{j=1}^{N}\th_j\sigma'(u_j^\sT x)u_j\nonumber\\
&=2(\th^\sT\sigma(Ux)-y)U^\sT\diag(\sigma'(Ux))\th\,.
\end{align}
As we work with Wasserstein of order $p=2$, we have conjugate order $q=2$. Therefore, Proposition \ref{pro:approx} gives us $V_{P_Z,q}(\ell)=\left(\E[  || \nabla_z\ell(\th;Z) ||_*^2 ]\right)^{1/2}$. By using~\eqref{eq:nonlinear:nabla_ell} we get
 \[ V_{P_Z,q}(\ell)=2\left(\E\left[ (\th^\sT\sigma(Ux)-y)^2\twonorm{U^\sT\diag(\sigma'(Ux))\th}^2  \right] \right)^{1/2} \,.
 \]
Finally, relation  $\AR(\th)=\SR(\th)+\eps V_{P_Z,q}(\ell)+O(\eps^2)$ from Proposition~\ref{pro:approx} completes the proof. 
We just need to verify that the necessary condition in Proposition \ref{pro:approx} holds for the loss $\ell(\th;z)= (y-\th^\sT\sigma(Wx))^2$. By the setting of the problem, we have $x\in \mathbb{S}^{d-1}(\sqrt{d})$ and $u_j\in \mathbb{S}^{d-1}(1)$. Therefore $\twonorm{x}\le \sqrt{d}$ and $\|U\|_{{\rm op}}\le \sqrt{\max(N,d)}$.

In the following lemma we show that the solution $\theta_\lambda$ to~\eqref{eq:weightedSum} is bounded as $\lambda$ varies in $[0,\infty)$.

\begin{lemma}\label{lem:thB}
Under the setting of Corollary \ref{coro:AR-nonlinear}, and for $\th_\lambda$ given by~\eqref{eq:weightedSum}, there exist constants $c_0$ and $c_1$, independent of $\lambda$, such that with probability at least $1-e^{-c_0 d}$ we have $\twonorm{\th_\lambda}\le c_1$.
\end{lemma}
Using Lemma~\ref{lem:thB} we can restrict ourselves to the ball of $\ell_2$ radius $c_1$.

We adopt the shorthands $D=\diag(\sigma'(Ux))$, $\tilde{D}=\diag(\sigma'(U\tilde{x}))$, $s=\sigma(Ux)$, and $\tilde{s}=\sigma(U\tilde{x})$, and write
\begin{align*}
	&\frac{1}{2}\|\nabla_z \ell(\th;z)-\nabla_{z} \ell(\th;\tilde{z})\|_* \\
	&=\frac{1}{2}\twonorm{\nabla_x \ell(\th;z)-\nabla_{x} \ell(\th;\tilde{z})}\\
	&\overset{(a)}{=}\twonorm{ (\th^\sT s-y)U^\sT D\th-  (\th^\sT \tilde{s}- \tilde{y})U^\sT \tilde{D}\th}\\
	&\overset{(b)}{\leq} \twonorm{\th^\sT sU^\sT(D-\tilde{D})\th}+ \twonorm{\th^\sT (s-\tilde{s})U^\sT \tilde{D} \th}+\twonorm{yU^\sT(D-\tilde{D})\th}+\twonorm{ (y-\tilde{y})U^\sT \tilde{D}\th }\\
   &\overset{(c)}{\leq}Nc_1^2 + \sqrt{N} c_1^2 \twonorm{s-\tilde{s}}+\sqrt{N} c_1^2+\sqrt{N} c_1 |y-\tilde{y}|\\
   &\overset{(d)}{\leq} (N+\sqrt{N})c_1^2 +N c_1^2 \twonorm{x-\tilde{x}} +\sqrt{N} c_1 |y-\tilde{y}|\\
   &\le (N+\sqrt{N})c_1^2+Nc_1^2 \left(\twonorm{x-\tilde{x}} +\infty\; \ind_{\{y\neq\tilde{y}\}}\right)\\
   &\overset{(e)}{\leq} M+ L \|z-\tilde{z}\|\,,
\end{align*}	
	where $(a)$ comes from~\eqref{eq:nonlinear:nabla_ell}, in $(b)$ we used triangle inequality, $(c)$ is a direct result of Cauchy inequality and the fact that $\sigma(u)\leq u$, $(d)$ comes from Lipschitz continuity of $\sigma$, and in $(e)$ we used $C=(N+\sqrt{N})c_1^2$ and $L=Nc_1^2$. Therefore the necessary condition in Proposition \ref{pro:approx} is satisfied.
	
	\subsubsection{Proof of Lemma~\ref{lem:thB}}
	By comparing the objective value \eqref{eq:weightedSum} at $\th_\lambda$ and $0$ and using the optimality of $\th_\lambda$ we get
	\begin{align*}
	&(1+\lambda) \SR(\theta_\lambda)\\
	&\le (1+\lambda) \SR(\theta_\lambda) + 2\eps\; \E_x\left[\big[(f_d(x)- \theta_\lambda^\sT \sigma(Ux) ) ^2 +\sigma^2\big] \twonorm{U^\sT \diag(\sigma'(Ux)) \th_\lambda}^2 \right]^{1/2}\\
	&\le (1+\lambda) \SR(0)\,.
	\end{align*}
	Therefore by invoking~\eqref{eq:SR-NL} we get
	\begin{align}
	 \E_{x}\left[  (f_d(x)- \theta_\lambda ^\sT \sigma(Ux) ) ^2\right]  \le \E_{x}\left[  f_d(x)^2\right]
	\end{align}
	Using the inequality $(a-b)^2\ge \frac{a^2}{2}-b^2$, we get
	\begin{align}\label{eq:UU0}
	\E[( \theta_\lambda ^\sT \sigma(Ux) )^2]\le 4 \E_x[f_d(x)^2] < c_2\,,
	\end{align}  
	with probability at least $1- e^{-c_3 d}$
	for some constants $c_2, c_3>0$. We next lower bound the eigenvalues of $\E[\sigma(Ux)  \sigma(Ux)^\sT ]$ from which we can upper bound $\twonorm{\th_\lambda}$.
	
	Define the dual activation of $\sigma$ as
	\[
	\tilde{\sigma}(\rho) = \E_{(v,w)\sim \normal_\rho} [\sigma(v)\sigma(w)]
	\]
	where $\normal_\rho$ denotes the two dimensional Gaussian with mean zero and covariance $\begin{pmatrix} 1 &\rho\\ \rho & 1\end{pmatrix}$. With this definition, we have
	$\E[(\sigma(Ux)  \sigma(Ux)^\sT)_{ij} ] = \tilde{\sigma}(u_i^\sT u_j)$ for $i, j =1,\dotsc, N$. Let $\{a_r\}_{r=0}^\infty$ denote the Hermite coefficients defined by
	\[
	a_r:= \frac{1}{\sqrt{2\pi}} \int_{-\infty}^\infty
\sigma(g) h_r(g) e^{-\frac{g^2}{2}} \de g\,,
\]
where $h_r(g)$ is the normalized Hermite polynomial defined by
\[
h_r(x) :=\frac{1}{\sqrt{r!}} (-1)^r e^{\frac{x^2}{2}} \frac{\de^r}{\de x^r} e^{-\frac{x^2}{2}}\,.
\]
Using the properties of normalized Hermite polynomials we have
\begin{align}
\tilde{\sigma}(\rho) = \E_{(v,w)\sim \normal_\rho} \Big[(\sum_{r=0}^\infty a_r h_r(v))(\sum_{\tilde{r}=0}^\infty a_{\tilde{r}} h_{\tilde{r}}(u))\Big] = \sum_{r=0}^\infty a_r^2 \rho^r. 
\end{align}
Writing in matrix form we obtain
\begin{align}\label{eq:UU1}
\E[(\sigma(Ux)  \sigma(Ux)^\sT)] = \tilde{\sigma}(UU^\sT) = \sum_{r=0}^\infty a_r^2 (UU^\sT)^{\odot r}\,,
\end{align}
where for a matrix $A^{\odot{r}} = A\odot (A^{\odot(r-1)})$ with $\odot$ denoting the Hadamard product (entrywise product). 

We next use the identity $(AA^\sT)\odot (BB^\sT) = (A* B)(A*B)^\sT$, with $*$ indicating the Khatri-Rao product. By using this identity and applying induction on $r$ it is straightforward to get the following relation for any matrix $A$:
\begin{align}\label{eq:r-kh-had}
 (AA^\sT)^{\odot r} = (A^{*r}) (A^{*r})^\sT\,,
 \end{align}
 with $A^{*{r}} = A*(A^{*(r-1)})$. 
By using the above identity in Equation~\eqref{eq:UU1} we obtain
\begin{align}
\E[(\sigma(Ux)  \sigma(Ux)^\sT)] = \sum_{r=0}^\infty a_r^2 (UU^\sT)^{\odot r} = \sum_{r=0}^\infty  (a_r U^{*r})(a_r U^{*r})^\sT \succeq a_r^2 (U^{*r}) (U^{*r})^\sT\,,
\end{align}
for any $r\ge 0$. Using this bound with $r=2$ and the fact that $a_2 = \frac{1}{2\sqrt{\pi}}$ for ReLU activation, we get
\begin{align}\label{eq:UU5}
\E[(\sigma(Ux)  \sigma(Ux)^\sT)] \succeq \frac{1}{4\pi} (U * U) \ge c_4\,,
\end{align}
where the last step holds with probability at least $1-e^{-c_5 d}$ for some constants $c_4$ and $c_5$ using the result of \cite[Corollary 7.5]{soltanolkotabi2018theoretical}.  

Combining Equations~\eqref{eq:UU0} and \eqref{eq:UU5} gives us $\twonorm{\theta_\lambda}\le \sqrt{c_2/c_4}$, which completes the proof.
\end{document}